\documentclass{article}

\usepackage[accepted]{icml2018}

\icmltitlerunning{Kernel Recursive ABC}

\usepackage{algorithmic}
\usepackage{amssymb}
\usepackage{graphicx}
\usepackage{amsmath}
\usepackage{mathrsfs}
\usepackage{adjustbox}

\usepackage[utf8]{inputenc} % allow utf-8 input
\usepackage[T1]{fontenc}    % use 8-bit T1 fonts
\usepackage{hyperref}       % hyperlinks
\usepackage{url}            % simple URL typesetting
\usepackage{booktabs}       % professional-quality tables
\usepackage{amsfonts}       % blackboard math symbols
\usepackage{nicefrac}       % compact symbols for 1/2, etc.
\usepackage{microtype}      % microtypography
\usepackage{amsmath,amsthm,amssymb}

\usepackage{subfigure} 
\usepackage{natbib}

\usepackage{comment,color}
\usepackage{bm}
\usepackage{wrapfig}
\usepackage{epsf}
\usepackage{natbib}

\newcommand{\Y}{{\mathcal{Y}}}

\renewcommand{\H}{{\mathcal{H}}}

\newcommand{\R}{\mathbb{R}}

\newtheorem{proposition}{Proposition} 
\newtheorem{remark}{Remark}
\newtheorem{corollary}{Corollary}

\newtheorem{assumption}{Assumption}

\newcommand{\argmax}{\operatornamewithlimits{argmax}}
\newcommand{\argmin}{\operatornamewithlimits{argmin}}

\begin{document}

\twocolumn[
\icmltitle{Kernel Recursive ABC: Point Estimation with Intractable Likelihood}



\icmlsetsymbol{equal}{*}

\begin{icmlauthorlist}
\icmlauthor{Takafumi Kajihara}{nec,aist}
\icmlauthor{Motonobu Kanagawa}{mpi}
\icmlauthor{Keisuke Yamazaki}{aist}
\icmlauthor{Kenji Fukumizu}{ism}
\end{icmlauthorlist}

\icmlaffiliation{nec}{NEC Corporation}
\icmlaffiliation{aist}{National Institute of Advanced Industrial Science and Technology}
\icmlaffiliation{mpi}{Max Planck Institute for Intelligent Systems}
\icmlaffiliation{ism}{The Institute of Statistical Mathematics}
\icmlcorrespondingauthor{Takafumi Kajihara}{t-kajihara@ct.jp.nec.com}

%\icmlkeywords{Machine Learning, ICML, Kernel method, ABC, MLE, maximum likelihood, intractable likelihood }

\vskip 0.3in

]

\printAffiliationsAndNotice{} 

\begin{abstract}
We propose a novel approach to parameter estimation for simulator-based statistical models with intractable likelihood.
Our proposed method involves recursive application of kernel ABC and kernel herding to the same observed data. 
We provide a theoretical explanation regarding why the approach works, showing (for the population setting) that, under a certain assumption, point estimates obtained with this method converge to the true parameter, as recursion proceeds.
We have conducted a variety of numerical experiments, including parameter estimation for a real-world pedestrian flow simulator, and show that in most cases our method outperforms existing approaches.
\end{abstract}

\section{Introduction} \label{sec:introduction}
Inference of parameters in a probabilistic model is an essential ingredient in model-based statistical approaches, both in the frequentist and Bayesian paradigms. 
Given a probabilistic model $P(y|\theta)$, which is a conditional distribution of observations $y$ given a parameter $\theta$, the aim is to make inference about the parameter $\theta^*$ that generated an observed data $y^*$.
When the model $P(y|\theta)$ admits a conditional density $\ell(y|\theta)$, such an inference can be made on the basis of evaluations of $\ell(y^*|\theta)$; this is the {\em likelihood} of $y^*$ as a function of $\theta$.
However, in modern scientific and engineering problems in which the model $P(y|\theta)$ is required to be sophisticated and complex, the likelihood function $\ell(y^*|\theta)$ might no longer be available.
This may be because the density form of $P(y|\theta)$ is elusive, or the evaluation of the likelihood $\ell(y^*|\theta)$ is computationally very expensive. 
Such situations, in which $\ell(y|\theta)$ (or $P(y|\theta)$) are referred to as {\em intractable likelihood}, make the inference problem quite challenging and are commonly found in the literature on population genetics \citep{pritchard1999population} and dynamical systems \citep{toni2009approximate}, to name just two.

{\em Approximate Bayesian Computation} (ABC) is a class of computational methods for Bayesian inference with intractable likelihood \citep{tavare1997inferring,pritchard1999population,beaumont2002approximate} that is applicable as long as {\em sampling} from the model $P(y|\theta)$ is possible.
Given a prior $\pi(\theta)$ on the parameter space, the basic ABC constructs a Monte Carlo approximation to the posterior $P_{y^*}(\theta) \propto P(y^*|\theta) \pi(\theta)$ in the following way:
i) sample pairs $(y_i,\theta_i)$ of pseudo data $y_i$ and parameter $\theta_i$ from the joint distribution $P(y|\theta)\pi(\theta)$, where $i = 1,\dots,n$ for some $n \in \mathbb{N}$, ii) maintain only those parameters $\theta_i$ associated with $y_i$ that are ``close enough'' to the observed data $y^*$, and iii) regard them as samples from the posterior $P_{y^*}(\theta)$.
ABC has been extensively studied in statistics and machine learning; see, e.g.,~ \citet{del2012adaptive,fukumizu2013kernel,meeds2014gps,pmlr-v51-park16,mitrovic2016dr}.

%Let $d$ be a dimension of $y$ and $\rho: \mathbb{R}^d \times \mathbb{R}^d \to \mathbb{R}$ be a distance function. 
%If $\rho(y,y^*)$ is less than a threshold $\epsilon \in \mathbb{R}$, the parameter $\theta$ is accepted as one of the posterior samples. Otherwise, another $\theta$ is newly sampled. 

In this paper, we rather take the frequentist perspective, and deal with the problem of {\em maximum likelihood estimation (MLE)} with intractable likelihood.
That is, we consider situations in which one believes that there is a ``true'' parameter $\theta^*$ that generated the data $y^*$ and wishes to obtain a point estimate for it.
This problem is also motivated by the following situations encountered in practice: 
1) Consider a situation in which the model is computationally expensive (e.g., a state-space model) and one wants to perform prediction based on it. 
In this case fully Bayesian prediction would require simulation from each of sampled parameters, which might be quite costly. 
If one has a point estimate of the true parameter $\theta^*$, then the computational cost can be drastically reduced.
2) Consider a situation in which one only has limited knowledge w.r.t. model parameters.
In this case, it is generally difficult to specify an appropriate prior distribution over the parameter space, and thus the resulting posterior may not be reliable.\footnote{For point estimation, one may think of using the maximum a posterior (MAP) estimate, but it may again be unreliable (as for the posterior distribution itself), if the prior distribution cannot be specified appropriately.}
%Moreover, as we show empirically in Sec \ref{sec:Gauss-misspecified}, existing ABC-based approaches to point estimation with intractable likelihood also suffer from the misspecification of a prior distribution.
%In such a situation, it is therefore useful to obtain a point estimate that is solely based on the observed data and is robust w.r.t. inadequate prior knowledge. 
Methods for point estimation with intractable likelihood have been reported in the literature, including the method of simulated-moments \citep{mcfadden1989method}, indirect inference \citep{gourieroux1993indirect}, ABC-based MAP estimation \citep{rubio2013simple}, noisy ABC-MLE \citep{dean2014parameter,yildirim2015parameter}, an approach based on Bayesian optimization \citep{gutmann2016bayesian}, and data-cloning ABC \citep{picchini2017approximate}.
We will discuss these existing approaches in Sec.~\ref{sec:experiments}.

Our contribution is in proposing a novel approach to point estimation with intractable likelihood on the basis of {\em kernel mean embedding of distributions}  \citep{MuaFukSriSch17}, a framework for statistical inference using reproducing kernel Hilbert spaces. 
Specifically, our approach extends {\em kernel ABC} \citep{fukumizu2013kernel,nakagome2013kernel}, a method for ABC using kernel embedding of conditional distributions \citep{song2009hilbert,SonFukGre13}, to point estimation with intractable likelihood.
The novelty lies in combining kernel ABC with {\em kernel herding} \citep{chen2010super}, a deterministic sampling method similar to quasi-Monte Carlo \citep{DicKuoSlo13}, and in applying these two methods iteratively to the same observed data in a recursive way.
We term this approach {\em kernel recursive ABC}.
A theoretical explanation will be provided for this approach, discussing how such recursion yields a point estimate for the true parameter.
We also discuss that the combination of kernel ABC and kernel herding leads to robustness against misspecification of a prior for the true parameter; this is an advantage over existing methods, and will be demonstrated experimentally.

This paper is organized as follows.
We briefly review kernel ABC and kernel herding in Sec.~\ref{sec:background} and propose kernel recursive ABC in Sec.~\ref{sec:proposed}.
We report experimental results of comparisons with existing methods in Sec.~\ref{sec:experiments}.
The experiments include parameter estimation for a real-world pedestrian flow simulator \citep{Noda2010p}, which may be of independent interest as application.

\section{Background} \label{sec:background}

\subsection{Kernel ABC} \label{sec:kernelABC}
Kernel ABC is an algorithm that executes ABC in a reproducing kernel Hilbert space (RKHS) and produces a reliable solution even in moderately large dimensional problems \citep{fukumizu2013kernel,nakagome2013kernel}. 
It is based on the framework of {\em kernel mean embeddings}, in which all probability measures are represented as elements in an RKHS (see \citet{MuaFukSriSch17} for a recent survey of this field). 
Let $\Theta$ be a measurable space, $k: \Theta \times \Theta \to \mathbb{R}$ be a measurable positive definite kernel, and $\mathcal{H}$ be its RKHS. 
In this framework, any probability measure $P$ on $\Theta$ will be represented as a Bochner integral
\begin{equation} \label{eq:kernel_mean}
\mu_P := \int_\Theta k(\cdot, \theta)dP(\theta) \in \mathcal{H},
\end{equation}
which is called the {\em kernel mean} of $P$.
If the mapping $P \to \mu_P$ is injective, in which case $\mu_P$ preserves all the information in $P$, the kernel $k$ is referred to as being {\em characteristic} \citep{fukumizu2008kernel}.
Characteristic kernels on $\Theta = \mathbb{R}^d$, for example, include Gaussian and Mat\'ern kernels \citep{SriGreFukSchetal10}.

Let $\mathcal{Y}$ be another measurable space and assume that an observed data $y^*\in \mathcal{Y}$ is provided. ($y^*$ is often a set of sample points.)
Given a conditional probability $P(y|\theta)$ and a prior $\pi(\theta)$, we wish to obtain the posterior distribution $P_{y^*}(\theta) \propto P(y^*|\theta)\pi(\theta)$.\footnote{There is abuse of notation here, as $P(y|\theta)$ does not denote a conditional density but a conditional distribution.}
As in a standard ABC, kernel ABC achieves this by first generating pairs of pseudo data and parameter $\{ (y_i,  \theta_{i}) \}_{i=1}^{n}$ from the joint distribution $P(y |\theta)\pi(\theta)$. 
It then estimates the kernel mean of the posterior $P_{y^*}$, which we denote by 
$$\mu_{P_{y^*}} := \int_\Theta k (\cdot,\theta)dP_{y^*}(\theta) \in \mathcal{H}.$$
Given a measurable positive definite kernel $k_{\mathcal{Y}}$ on $\mathcal{Y}$, the estimator is given by
\begin{eqnarray}
\hat{\mu}_{P_{y^*}} &=& \sum^{n}_{i=1}w_i k(\cdot , \theta_i) \in \mathcal{H}, \label{eq:cond_kmean} \\
{\bm w} &:=& (w_1,\dots,w_n)^T := (G + n \delta I)^{-1} {\bm k}(y^*), \label{eq:KRR_weight}
\end{eqnarray}
where ${\bm k}(y^*) := (k_\mathcal{Y}(y_1,y^*), \dots, k_\mathcal{Y}(y_n, y^*))^T \in \mathbb{R}^n$, $G := (k_\mathcal{Y}(y_i,y_j))_{i,j=1}^n \in \mathbb{R}^{n \times n}$, $\delta > 0$ is a regularization constant, and $I \in \mathbb{R}^{n \times n}$ is an identity matrix. 
The estimator (\ref{eq:cond_kmean}) is essentially an (RKHS-valued) kernel ridge regression \citep{GruLevBalPatetal12}: Given {\em training} data $\{ (y_i,  k(\cdot,\theta_{i}) ) \}_{i=1}^{n}$, the weights (\ref{eq:KRR_weight}) provide an estimator for the mapping $y^* \Rightarrow k(\cdot,\theta^*)$.
For consistency and convergence results, which require $\delta \to 0$ as $n \to \infty$, we re refer to \citet{fukumizu2013kernel} and \citet{MuaFukSriSch17}.

\subsection{Kernel herding} \label{sec:kernelherding}

Kernel herding is a deterministic sampling technique based on the kernel mean representation of a distribution \citep{chen2010super} and can be seen as a greedy approach to quasi-Monte Carlo \citep{DicKuoSlo13}.
Consider sampling from $P$ using the kernel mean $\mu_P$ (\ref{eq:kernel_mean}), and assume that one is able to evaluate function values of $\mu_P$.
Kernel herding greedily obtains sample points $\theta_1,\theta_2,\dots, \theta_n$ by iterating the following steps: Defining $h_{0} := \mu_P$, 
\begin{eqnarray}
\theta_{t+1} &=& \argmax_{\theta \in \Theta}  h_t(\theta), \label{eq:herding1}\\
h_{t+1} &=& h_t + \mu_P - k(\cdot , \theta_{t+1}) \in \mathcal{H}, \label{eq:herding2}
\end{eqnarray}
where $t = 0,\dots,n-1$.
\citet{chen2010super} has shown that, if there exists a constant $C > 0$ such that $k(\theta, \theta) = C$ for all $\theta \in \Theta$, this procedure will be identical to the greedy minimization of the maximum mean discrepancy (MMD) \citep{gretton2007kernel,gretton2012kernel}:
\begin{equation} \label{eq:MMD}
\epsilon_n := \left\| \mu_P - \frac{1}{n}\sum^n_{t=1} k(\cdot, \theta_t) \right\|_\mathcal{H}, 
\end{equation}
where $\| \cdot \|_\mathcal{H}$ denotes the norm of $\H$.
That is, the points $\theta_1,\dots,\theta_n$ are obtained so as to (greedily) minimize the distance $\varepsilon_n$ between $\mu_P$ and the empirical kernel mean $\frac{1}{n}\sum^n_{t=1}k(\cdot , \theta_t)$.
The generated points $\theta_1,\dots,\theta_n$ are also called super-samples because they are more informative than those from random sampling; this is in the sense that error decreases at the rate $\epsilon_{n} = O(n^{-1})$ if the RKHS is finite-dimensional \citep{bach2012equivalence}, which is faster than the rate $\epsilon_{n} = O(n^{-1/2})$ of random sampling \citep{SmoGreSonSch07}.
%If the RKHS is infinite-dimensional, the known best rate is $\epsilon_{n} = O(n^{-1/2})$. However, empirical studies suggest that the same rate for a finite-dimensional RKHS would also hold. \citep{bach2012equivalence}.
Convergence guarantees are also provided even when the optimization problem in (\ref{eq:herding1}) is solved approximately \citep{LacLinBac15} and when the kernel mean $\mu_P$ is replaced by an empirical estimate $\hat{\mu}_P$ of the form (\ref{eq:cond_kmean})  \citep{kanagawa2016filtering}.
Note that the decay $\epsilon_n \to 0$ of the error (\ref{eq:MMD}) as $n \to \infty$ implies the convergence of expectation $\frac{1}{n} \sum_{t=1}^n f(\theta_t) \to \int f(x)dP(x)$ for all functions $f$ in the RKHS $\H$ and for functions $f$ that can be approximated well by the RKHS functions \citep{KanSriFuk16}.
%Roughly speaking, the reason for the fast convergence is that kernel herding is mode-seeking in the distribution and optimal with respect to reducing the error in the RKHS at every iteration. 
%Additionally, owing to the characteristics of the super-samples from kernel herding, we may expect kernel herding to suggest where in the parameter space to explore at each iteration of the recursive application of the Bayes rule, so as to reach a more reliable point estimate.

\section{Proposed method} \label{sec:proposed}
%%%
%\subsection{The idea and procedure} 
\label{sec:proposed-algorithm}

\begin{algorithm}[tb]
   \caption{Kernel Recursive ABC}
   \label{kr-abc}
\begin{algorithmic}
   \STATE {\bfseries Input:} A prior distribution $\pi$, an observed data $y^*$, a data generator $P(y|\theta)$, the number  $N_{\rm iter}$ of iterations, the number $n$ of simulated pairs, a kernel $k$ on $\Theta$, a kernel $k_\mathcal{Y}$ on $\mathcal{Y}$, and a regularization constant $\delta > 0$
   \STATE {\bfseries Output:} A point estimate $\acute{\theta}$.
   \FOR{\texttt{$N = 1,..., N_{\rm iter}$}}
   \IF{$N=1$}{
    	\FOR{\texttt{$i = 1,..., n$}}
        \STATE
    	Sample $\theta_{1,i} \sim \pi (\theta)$ i.i.d.
        \ENDFOR
        }    
    \ENDIF 
    \FOR{\texttt{$i = 1,..., n$}}
    \STATE Generate $y_{N,i} \sim P(\cdot | 		\theta_{N,i})$
    \ENDFOR   
    \STATE Compute $G := (k_\mathcal{Y}(y_{N,i}, y_{N,i}))_{i,j=1}^n \in \mathbb{R}^{n \times n}$ and ${\bm k}(y) := (k_\mathcal{Y}(y_{N,i}, y^*))_{i=1}^n \in \mathbb{R}^n$.
    \STATE Calculate ${\bm w} = (w_1,\dots,w_n)^T \in \mathbb{R}^n$ by Eq.(\ref{eq:KRR_weight}).
    \STATE Construct a kernel mean estimate of the powered posterior $\hat{\mu}_{P_N} := \sum_{i=1}^{n} w_i k(\cdot, \theta_{N,i})$
    \STATE Sample $\{\theta_{N+1,t}\}_{t=1}^n$ by performing kernel herding Eqs.(\ref{eq:herding1}) (\ref{eq:herding2}) with $\mu_P := \hat{\mu}_{P_N}$. 
\ENDFOR
\STATE Obtain a point estimate $\acute{\theta} := \theta_{N_{\rm iter}+1,1}$
\end{algorithmic}
\end{algorithm}

\begin{comment}
\begin{algorithm}[t]    
\caption{Kernel Recursive ABC}
	\label{kr-abc}  
    \SetKwInOut{Input}{Input}
    \SetKwInOut{Output}{Output}
    \Input{A prior distribution $\pi$, an observed data $y^*$, a data generator $P(y|\theta)$, the number  $N_{\rm iter}$ of iteration, the number $n$ of input-output pairs, a kernel $k$ on $\Theta$, a kernel $k_\mathcal{Y}$ on $\mathcal{Y}$ and a regularization constant $\delta > 0$.}
    \Output{A maximum likelihood estimate $\acute{\theta}$.}
    \For{\texttt{$N = 1,..., N_{\rm iter}$}}{
    \If{N=1}{
    	\For{\texttt{$i = 1,..., n$}}
        {
    	Sample $\theta_{1,i} \sim \pi (\theta)$ i.i.d.
      }
        }
    \For{\texttt{$i = 1,..., n$}}{
        Generate $y_{N,i} \sim P(\cdot | \theta_{N,i})$}
    Compute $G := (k_\mathcal{Y}(y_{N,i}, y_{N,i}))_{i,j=1}^n \in \mathbb{R}^{n \times n}$ and ${\bm k}(y) := (k_\mathcal{Y}(y_{N,i}, y^*))_{i=1}^n \in \mathbb{R}^n$, and calculate ${\bm w} = (w_1,\dots,w_n)^T \in \mathbb{R}^n$ by Eq.(\ref{eq:KRR_weight}). \\
    Construct a kernel mean estimate of the powered posterior $\hat{\mu}_{P_N} := \sum_{i=1}^{n} w_i k(\cdot, \theta_{N,i})$\\
    Sample $\{\theta_{N+1,t}\}_{t=1}^n$ by performing kernel herding Eqs.(\ref{eq:herding1}) (\ref{eq:herding2}) with $\mu_P := \hat{\mu}_{P_N}$. }
Obtain a point estimate $\acute{\theta} := \theta_{N_{\rm iter}+1,1}$
\end{algorithm}
\end{comment}
Our idea is to recursively apply Bayes' rule to the same observed data $y^*$ by using the posterior obtained in one iteration as a prior for the next iteration.
For this, let $\ell(\theta) := \ell(y^*|\theta)$ be a likelihood function and $\pi(\theta)$ be a prior density, where $\theta \in \Theta$, with $\Theta$ being a measurable space.
Consider the population setting in which no estimation procedure is involved. 
After the $N$-th recursion, the posterior distribution becomes
\begin{equation}  \label{eq:power_post} 
p_N (\theta) := C_N^{-1}\ \pi(\theta) (\ell(\theta))^N,
\end{equation}
where $C_N := \int_\Theta \pi(\theta) \left( \ell(\theta) \right)^N d\theta$ is a normalization constant.
We refer here to this as a {\em powered posterior}.
If $\ell$ has a unique global maximum at $\theta_{\infty} \in \Theta$ and the support of $\pi$ contains $\theta_\infty$, one can show that $p_N$ converges weakly to the Dirac distribution $\delta_{\theta_{\infty}}$ at $\theta_\infty$ under certain conditions \citep{LelNadSch10}.
In other words, the effect of the prior diminishes as the recursion proceeds, and the powered posterior degenerates at the maximum likelihood point, providing a method for MLE.
A similar idea has been discussed by \citet{doucet2002marginal,LelNadSch10} in the context of data augmentation and data cloning, in which one replicates the observed data $y^*$ multiple times and applies Bayes' rule once; our approach is different, as we employ recursive applications of Bayes' rule multiple times (this turns out to be beneficial in our approach, as is shown below).

Based on the above idea, we propose to recursively applying kernel ABC (Sec.~\ref{sec:kernelABC}) and kernel herding (Sec.~\ref{sec:kernelherding}).
Specifically, the proposed method (Algorithm \ref{kr-abc}) iterates the following procedures: 
(i) At the $N$-th iteration, the kernel mean $\mu_{P_N} := \int k(\cdot,\theta)p_N(\theta)d\theta$ of the powered posterior (\ref{eq:power_post}) is estimated using simulated pairs $\{(\theta_{N,i},y_{N,i})\}_{i=1}^n$ via kernel ABC; 
(ii) from the estimate $\hat{\mu}_{P_N}$ of $\mu_{P_N}$ given in (i), new parameters $\{ \theta_{N+1,i} \}_{i=1}^n$ are generated via kernel herding, and new pseudo-data $\{ y_{N+1,i} \}_{i=1}^n$ are generated from the simulator $P(y_{N+1,i}|\theta_{N+1,i})$ in the $N+1$-th iteration.
After iterating these procedures $N_{\rm iter}$ times, point estimate $\acute{\theta}$ for the true parameter is given as the first point $\theta_{N_{\rm iter}+1,1}$ from kernel herding at the last iteration.

\vspace{-0.3cm}
\paragraph{Auto-correction mechanism.}
An interesting feature of the proposed approach is that, as experimentally indicated in Sec.~\ref{sec:Gauss-misspecified}, it is equipped with an auto-correction mechanism: If the parameters $\theta_{N,1},\dots,\theta_{N,n}$ at the $N$-th iteration are far apart from the true parameter $\theta^*$, then Algorithm \ref{kr-abc} searches for the parameters $\theta_{N+1,1},\dots,\theta_{N+1,n}$ at the next iteration, so as to explore the parameter space $\Theta$.
For instance, if the prior $\pi(\theta)$ is misspecified, meaning that the true parameter $\theta^*$ is not contained in the support of $\pi(\theta)$, then the initial parameters $\theta_{1,1},\dots,\theta_{1,n}$ from  $\pi(\theta)$ are likely to be apart from the true parameter $\theta^*$.
The auto-correction mechanism makes the proposed method robust to such misspecification and makes it suitable for use in situations in which one lacks appropriate prior knowledge about the true parameter.

To explain how this works, let us explicitly write down the procedure (\ref{eq:herding1}) (\ref{eq:herding2}) of kernel herding as used in Algorithm \ref{kr-abc}.
Given that $t\ (< n)$ points $\theta_{N+1,1},\dots,\theta_{N+1,t}$ have already been generated, the next point $\theta_{N+1,t+1}$ is obtained as 
\begin{eqnarray} \label{eq:herding_alg}
&&\theta_{N+1,t+1} := \\ 
&& \argmax_{\theta \in \Theta} \sum_{i=1}^{n} w_i k(\theta, \theta_{N,i}) -  \dfrac {1}{t+1}\sum^t_{i=1}k(\theta , \theta_{N+1,i}), \nonumber 
\end{eqnarray}
where the weights $w_1,\dots,w_n$ are given as (\ref{eq:KRR_weight}).
Assume that all the simulated parameters $\theta_{N,1},\dots,\theta_{N,n}$ at the $N$-th iteration are far apart from the true parameter $\theta^*$: If $N=1$, these are the parameters sampled from the prior $\pi(\theta)$.
Then it is likely the resulting simulated data $y_{N,1},\dots,y_{N,n}$ are dissimilar to the observed data $y^*$.
In this case, each component of the vector ${\bm k}(y) := (k_\mathcal{Y}(y_{N,i}, y^*))_{i=1}^n \in \mathbb{R}^n$  becomes nearly $0$, since $k_\mathcal{Y}(y_{N,i}, y^*)$ quantifies the similarity between $y^*$ and $y_{N,i}$.
As a result, each of the weights $w_1,\dots,w_n$ given by kernel ABC (\ref{eq:KRR_weight}) also become nearly $0$, and thus the first term on the right side in (\ref{eq:herding_alg}) will be ignorable.
The point $\theta_{N+1,{t+1}}$ is then obtained so as to roughly maximize the second term $-\frac{1}{t+1} \sum_{i=1}^t k(\theta_{N+1,{t+1}},\theta_{N+1,i})$, or, equivalently, so as to minimize $ \sum_{i=1}^t k(\theta_{N+1,{t+1}},\theta_{N+1,i})$.
Since the kernel $k(\theta_{N+1,{t+1}}, \theta_{N+1,i})$ measures the similarity between $\theta_{N+1,{t+1}}$ and $\theta_{N+1,i}$, the new point $\theta_{N+1,{t+1}}$ is located apart from the points $\theta_{N+1,1},\dots,\theta_{N+1,t}$ generated so far.
In this way, the parameters $\theta_{N+1,1},\dots,\theta_{N+1,n}$ at the $N+1$-th iteration are made to explore the parameter space $\Theta$ if parameters $\theta_{N,1},\dots,\theta_{N,n}$ at the $N$-th iteration are far apart from the true parameter $\theta^*$.

\subsection{Theoretical analysis} \label{sec:theory}
We provide here a theoretical basis for the proposed recursive approach.
Since the consistency of kernel ABC and kernel herding have already been established in the literature \citep{fukumizu2013kernel,bach2012equivalence}, we focus  on convergence analysis in the population setting, that is, convergence analysis for the kernel mean of the powered posterior (\ref{eq:power_post}) and for the resulting point estimate.
We nevertheless note that convergence analysis of the overall procedure of Algorithm \ref{kr-abc} remains an important topic for future research.
All the proofs can be found in the Supplementary Materials.

Below, we let $\Theta$ be a Borel measurable set in $\R^d$.
Denote by $P_N$ the probability measure induced by the powered posterior density $p_N$ (\ref{eq:power_post}), and let $\mu_{P_N} := \int k(\cdot,\theta)dP_N(\theta) \in \mathcal{H}$ be its kernel mean, where $k$ is a kernel on $\Theta$ and $\mathcal{H}$ is its RKHS. 
We require the following assumption for the likelihood function $\ell$ and the prior $\pi$ for theoretical analysis.
\begin{assumption} \label{as:likelihood_prior}
(i) $\ell$ has a unique global maximum at $\theta_\infty \in \Theta$, and $\pi(\theta_\infty) > 0$;
(ii) $\pi$ is continuous at $\theta_\infty$, $\ell$ has continuous second derivatives in the neighborhood of $\theta_\infty$, and the Hessian of $\ell$ at $\theta_\infty$ is strictly negative-definite.
\end{assumption}

Our first result below shows that, under Assumption \ref{as:likelihood_prior}, the powered posterior $P_N$ (\ref{eq:power_post}) converges to the Dirac distribution $\delta_{\theta_\infty}$ in the RKHS $\mathcal{H}$ as $N \to \infty$; 
this provides a theoretical basis for recursively applying the kernel ABC.
\begin{proposition} \label{prop:convergence}
Let $\Theta \subset \R^d$ be a Borel measurable set and $k: \Theta \times \Theta \to \R$ be a continuous bounded kernel.
Under Assumption \ref{as:likelihood_prior}, we have 
%\begin{eqnarray*}
$\lim_{N \to \infty} \left\| \mu_{P_N} - k(\cdot,\theta_\infty) \right\|_{\mathcal{H}} = 0$.
%\end{eqnarray*}
\end{proposition}
Proposition \ref{pr:consistency2} below provides a justification for the use of the first point of kernel herding (here this is $\theta_N := \argmin_{\tilde{\theta} \in \Theta} \| \mu_{P_N} - k(\cdot,\tilde{\theta}) \|_{\mathcal{H}}$; see Sec.~\ref{sec:kernelherding}) as a point estimate of $\theta_\infty$.
To this end, we introduce the following assumption on the kernel, which is satisfied by, for example, Gaussian and Mat\'ern kernels.
\begin{assumption} \label{as:kernel}
(i) There exists a constant $C > 0$ such that $k(\theta,\theta) = C$ for all $\theta \in \Theta$.
(ii)  It holds that $k(\theta,\theta') < C$ for all $\theta, \theta' \in \Theta$ with $\theta \neq \theta'$.
\end{assumption}

\begin{proposition}
\label{pr:consistency2}
Let $\Theta \subset \R^d$ be a compact set, and $k: \Theta \times \Theta \to \R$ be a continuous, bounded kernel. 
Let $\theta_N := \argmin_{\tilde{\theta} \in \Theta} \left\lVert \mu_{P_N} - k(\cdot,\tilde{\theta}) \right\rVert_{\mathcal{H}}$.
If Assumptions \ref{as:likelihood_prior} and \ref{as:kernel} hold, then we have $\theta_N \to \theta_{\infty}$ as $N \to \infty$.
\end{proposition}

We make a few remarks regarding Assumption \ref{as:likelihood_prior}.
The assumption that $\ell$ has a unique global maximum is not satisfied if the model is singular, an example being mixture models: In this case there are multiple global maximums.
However, our experiment in Sec.~\ref{sec:gaussian-mixture} shows that even for mixture models, the proposed method works reasonably well.
This suggests that, in an empirical setting, a point estimate may converge to one of the global maximums.
The assumption $\pi(\theta_\infty) > 0$ will also not be satisfied if the support $\pi$ does not contain $\theta_\infty$, but the proposed method performs well even in this case (as shown in \ref{sec:Gauss-misspecified}), possibly thanks to the auto-correction mechanism explained above. 
We reserve further analysis of these properties for future work.

%If true, Proposition 2 but not Corollary 1 should be sufficient for consistency with MLE. We have allowed the both of them to stay at this juncture, just in case.  

\section{Experiments} \label{sec:experiments}
We have conducted a variety of experiments comparing the proposed method with existing approaches. 
We begin with a quick review of these approaches (Sec.~\ref{sec:existing-methods}), and report experimental results on point estimation with a misspecified prior (Sec.~\ref{sec:Gauss-misspecified}), population dynamics of the blowfly (Sec.~\ref{sec:blowfly}), alpha stable distributions (Sec.~\ref{sec:alpha-stable}), Gaussian mixture models with redundant components (Sec.~\ref{sec:gaussian-mixture}), and a real-world pedestrian simulator (Sec.~\ref{sec:pedestrian-sim}). 

%as kernel ABC \citep{nakagome2013kernel}, K2-ABC \citep{pmlr-v51-park16}, adaptive SMC-ABC \citep{del2012adaptive}, data-cloning ABC (ABC-DC) \citep{picchini2017approximate}, a method based on Bayesian optimization \citep{gutmann2016bayesian}, and the method of simulated moments (MSM) \citep{mcfadden1989method}. 

\subsection{Existing approaches and experimental settings} \label{sec:existing-methods}
{\bf K2-ABC} \citep{pmlr-v51-park16} is an ABC method that represents the empirical distributions of simulated and test observations as kernel means in an RKHS.
For each of simulated parameters, the associated weight is calculated by using the RKHS distance between the kernel means (i.e., MMD), and the resulting weighted sample is treated as a posterior distribution.
{\bf Adaptive SMC-ABC} \citep{del2012adaptive} is a rejection-based approach based on sequential Monte Carlo, which sequentially updates the tolerance level and the associated proposal distribution in an adaptive manner.
This method is a state-of-the-art ABC approach.
%Notably, computational complexity is linear with respect to the number of particles. 
%{\bf Noisy ABC-MLE} \citep{yildirim2015parameter} is a gradient-based method for MLE with intractable likelihood. This method assumes that sampling from an intractable model can be achieved by deterministic mapping applied to a simple random variable, and that the gradient of the deterministic mapping is available. Therefor,e this method can only be applied to a limited class of generative models, though for such models it can perform well.
%It approximates likelihoods in a manner similar to that of ABC. 
%perturbed??? the observed and pseudo-data to ensure the consistency for MLE. 
The approach by \citet{gutmann2016bayesian}, which we refer to as {\bf Bayesian Optimization} for simplicity, is a method for MLE with intractable likelihood based on Bayesian optimization \citep{brochu2010tutorial}. 
This method optimizes the parameters in a intractable model so as to minimize the discrepancy between the simulated and test observations. 
Note that comparison with this method in terms of computation time may not make sense (although we report them for purposes of completeness), as we used publicly available code\footnote{https://sheffieldml.github.io/GPyOpt/} for implementation.
{\bf The method of simulated moments (MSM)} \citep{mcfadden1989method} optimizes the parameter in the model so that the resulting moments of simulated data match those of observe data.  
MSM may be seen a special case of indirect inference \citep{gourieroux1993indirect}, an approach studied in econometrics.\footnote{MSM is a special case of indirect inference because the moments can be regarded as the parameters of an auxiliary model.}
{\bf Data-cloning ABC (ABC-DC)} \citep{picchini2017approximate} is an approach combining ABC-MCMC \citep{marjoram2003markov} and Data Cloning \citep{LelNadSch10}, replicating observed data multiple times to achieve MLE with intractable likelihood.

\vspace{-0.5cm}
\paragraph{Experimental settings.}
Unless otherwise specified, the following settings were applied in the experiments.
For all the methods that employed kernels, we used Gaussian kernels.
The discrepancy between the simulated and observed data was measured by the {\em energy distance} \citep{szekely2013energy}, which is a standard metric for distributions in statistics and can be computed only from pairwise Euclidean distances between data points.
Since the usual quadratic time estimator was too costly, we used a linear time estimator for computing the energy distance (see the Supplementary Materials for details).

For each method, unless otherwise specified, we determined the hyper-parameters on the basis of the cross-validation-like approach described in \citet[Sec.~4]{pmlr-v51-park16}.
That is, to evaluate one configuration of hyper-parameters, we first used 75\% of the observed data for point estimation and then computed the discrepancy between the rest of the observed data and the ones simulated from point estimates; after applying this procedure to all candidate configurations, the one with the lowest discrepancy was finally selected.
The bandwidth of a Gaussian kernel was selected from candidate values, each of which is the median (of pairwise distances) multiplied by logarithmically equally spaced values between $2^{-4}$ and $2^4$ \citep[Sec.~5.1.1]{takeuchi2006nonparametric}.
Regularization constants for the proposed method and kernel ABC, as well as the soft threshold for K2-ABC, were selected from logarithmically spaced values between $10^{-4}$ and $1$. 
To compute MMD for K2-ABC, a linear time estimator \citep[Sec.~6]{gretton2012kernel} was used to reduce computational time, as the usual quadratic time estimator was too costly.
For Adaptive SMC-ABC, the initial tolerance level was set as the median of pairwise distances between the observed and simulated data. 
For Bayesian Optimization, we used Expected Improvement as an acquisition function, and all the hyper-parameters were marginalized out following the approach of \citet[Sec.~3.2]{snoek2012practical}.
For MSM, the number of moments were selected from a range up to 30 by the cross-validation like approach. 
For ABC-DC, we employed, in particular, dynamic ABC-DC, which automatically adjusts its associated parameters.
To obtain point estimates with kernel ABC and K2-ABC, we computed the means of the resulting posterior distributions. 
For Adaptive SMC-ABC, point estimates were obtained as posterior means as well as MAP estimates by applying the mean shift algorithm to posterior weighted samples \citep{fukunaga1975estimation}, the latter essentially being an approach suggested by \citet{rubio2013simple}.
%For kernel recursive ABC, the point estimate was the first sample of kernel herding after the last iteration. 

The following abbreviations may be used for the sake of simplicity; kernel recursive ABC is referred to as {\bf KR-ABC}, kernel ABC as {\bf K-ABC}, adaptive SMC-ABC as {\bf SMC-ABC}, Bayesian Optimization as {\bf BO}, and Dynamic ABC-DC as {\bf ABC-DC}.
For our method, we also report results based on a half number of iterations, which we call {\bf KR-ABC (less)}.
% since??? we are aware of that the computational time could be the issue.

\subsection{Multivariate Gaussian distribution with a severely misspecified prior} \label{sec:Gauss-misspecified}
As a proof of concept regarding the auto-correction mechanism of the proposed method described in Sec.\ref{sec:proposed-algorithm}, we have performed an experiment for when the prior distribution is severely misspecified (see the Supplementary Materials for an illustration). 
The task is to estimate the mean vector of a 20-dimensional Gaussian distribution ${\rm Normal}(\mu,\Sigma)$, where the true mean vector is $\mu := (10,50,90,130,180,280,390,430,
520,630,1010,1050,\\ 1090,1130,1180,1280,1390,1430,1520,1630)^T \in \R^{20}$. 
The covariance matrix $\Sigma \in \R^{20 \times 20}$ is assumed to be known and is a diagonal matrix with all diagonal elements being $40$.
%was??? set to $Uniform(9000000,10000000)$. 
Test data $y^*$ consisted of 100 i.i.d.~observations from this Gaussian distribution. 
As a prior for the mean vector $\mu$, we used the uniform distribution on $[9\times 10^6, 10^7]^{20}$, which is extremely misspecified.
For Bayesian optimization, the space to be explored was set as $[0, 10^7]^{20}$.  
%The energy distance \citep{szekely2013energy} was used to measure the discrepancy between the empirical distributions of the observed and pseudo-data. 

In this experiment, each pseudo data was made from 100 observations simulated with one parameter configuration.
%Comparisons were made with K2-ABC, K-ABC, SMC-ABC, ABC-DC, BO and MSM.
K2-ABC and K-ABC used 3000 pairs of a parameter and pseudo-data.
For the proposed method and SMC-ABC, we generated 100 pairs of a parameter and pseudo-data for the initial iteration, and then the iterations were repeated 30 times, resulting in a total of 3000 simulations. 
For the proposed method, the bandwidth of the kernel $k_{\Y}$ on observed data was recomputed for each iteration, using the median heuristic. 
%For kernel recursive ABC, the kernel parameter for the kernel ABC part was recomputed to be the median of the distance between each of the simulated data items. 
For SMC-ABC, the parameter $\alpha \in (0,1)$, which controls the trade-off between the speed of convergence and the accuracy of posterior approximation, was set to be 0.3, as we found this value to be the best in terms of the trade-off.

For each method, we ran 30 independent trials, and the results in averages and standard deviations are shown in Table \ref{tab:table_mvincorrect}, where the {\em parameter error} is the mean (over 20 dimensions) of the absolute difference between the estimated and the true parameter values divided by the true value, and the {\em data error} is the energy distance between the true data and pseudo data simulated with the estimated parameter. 
Surprisingly, the proposed method successfully approached the true parameter even when the prior was severely misspecified.
As discussed in Sec~\ref{sec:proposed-algorithm} and demonstrated in the Supplementary Materials, this would appear to be because of the use of kernel herding, which automatically widens the space to explore when simulated data is far apart from test data.
As expected, other methods were unable to approach the true parameter.

%The results show that only the proposed algorithm and MSM were able to approach the true parameters. Other Bayesian methods were unable to avoid the support of the prior distribution.
\begin{comment}
\begin{table}[t]
  \caption{Results for the pedestrian simulator in Sec.~\ref{sec:pedestrian-sim}} \label{tab:table3}
  \begin{center}
  \begin{adjustbox}{max width=230pt}
  \begin{tabular}{|l|c|r|r|r|r|} \hline
    Algorithm & $\theta^{(N)}$ error & $\theta^{(T)}$ error & data error& cputime \\ \hline
    KR-ABC & 61.58(74.42) & 70.93(102.08) & 0.008(0.009) & 2233.45(97.54)\\ \hline
    KR-ABC (less) & 82.46(75.05) & 134.00(161.85) & 0.014(0.014) & 1875.32(147.16)\\ \hline
    K2-ABC & 298.94(120.71) & 308.95(109.43) & 0.10(0.10)& 1547.32(56.31)\\ \hline
    K-ABC & 354.72(145.76) & 389.52(140.91) & 0.12(0.09)& 1773.74(84.91)\\ \hline
    SMC-ABC (mean) & 271.51(104.64) & 363.12(91.28) & 0.09(0.07)& 2017.89(110.02)\\ \hline
    SMC-ABC (MAP) & 255.15(139.33) & 348.43(104.74) & 0.09(0.1)& 2017.89(110.02)\\ \hline
    ABC-DC & 273.93(136.14)& 327.48(98.12) & 0.09(0.14) & 1984.43(59.12)\\ \hline
    BO & 194.57(65.83) & 291.73(105.33) & 0.04(0.06) & 37541.23(3047.46)\\ \hline
    MSM & 453.58(89.43) & 510.04(55.10) & 0.24(0.17) & 1869.83(49.51)\\ \hline

  \end{tabular}
  \end{adjustbox}
  \end{center}
  \vspace{-0.25in}
\end{table}
\end{comment}

\begin{table}[t]
  \caption{Results for multivariate Gaussian distributions in Sec.~\ref{sec:Gauss-misspecified}}
    \begin{center}
   \begin{adjustbox}{max width=230pt,
   max height = 100pt}
  \begin{tabular}{|l|c|r|r|r|} \hline
    Algorithm & parameter error & data error& cputime \\ \hline 
    KR-ABC & 0.70(0.29) & 0.008(0.004) & 866.02(26.12) \\ \hline
    KR-ABC (less) & 7.22(3.28) & 0.02(0.24) & 353.498(23.05) \\ \hline
    K2-ABC & >1e+6 (>1e+3) & >1e+5 (>1e+3) & 209.51(11.49) \\ \hline
    K-ABC & >1e+6 (>1e+3) &>1e+5 (>1e+3)) & 403.93(24.97) \\ \hline
    SMC-ABC (mean) & >1e+6 (>1e+3) & >1e+5 (>1e+3) & 590.41(29.54)\\ \hline
	SMC-ABC (MAP) & >1e+6 (>1e+3) & >1e+5 (>1e+3) & 590.41(29.54)\\ \hline
    ABC-DC & >1e+6 (>1e+3) & >1e+5 (>1e+3) & 313.99(16.85)\\ \hline
    BO & >1e+5(>1e+4) & >1e+5 (>1e+4)&  25940.86(936.40)\\ \hline
    MSM & >1e+5(>1e+4) & >1e+5(>1e+4)& 307.42(67.94)\\ \hline
  \end{tabular}
  \end{adjustbox}
  \label{tab:table_mvincorrect}
  \end{center}
    \vspace{-0.15in}
\end{table}

\begin{comment}

\begin{table}[h]
  \caption{Error on the experiment of multivariate Gaussian distribution.}
    \begin{center}
   \begin{adjustbox}{max width=230pt}
  \begin{tabular}{|l|c|r|r|r|} \hline
    Algorithm & parameter error & data error& cputime \\ \hline 
    KR-ABC & 0.007(0.003) & 0.008(0.004) & 866.02(26.12) \\ \hline
    KR-ABC (less) & 0.04(0.32) & 0.02(0.24) & 353.498(23.05) \\ \hline
    K2-ABC & >1e+5 (>1e+3) & >1e+5 (>1e+3) & 209.51(11.49) \\ \hline
    K-ABC & >1e+5 (>1e+3) &>1e+5 (>1e+3)) & 403.93(24.97) \\ \hline
    SMC-ABC (mean) & >1e+5 (>1e+3) & >1e+5 (>1e+3) & 590.41(29.54)\\ \hline
	SMC-ABC (MAP) & >1e+5 (>1e+3) & >1e+5 (>1e+3) & 590.41(29.54)\\ \hline
    ABC-DC & >1e+5 (>1e+3) & >1e+5 (>1e+3) & 313.99(16.85)\\ \hline
    BO & 2229.66(779.20) & >1e+5 (>1e+3)&  25940.86(936.40)\\ \hline
    MSM & 0.913(0.12) & 900.37(3.47)& 508.03(48.13)\\ \hline
  \end{tabular}
  \end{adjustbox}
  \label{tab:table_mvincorrect}
  \end{center}
\end{table}

\begin{figure}[t]
\vspace{.3in}
\begin{minipage}{1.0\hsize}
\begin{center}
\includegraphics[scale = 0.3]{incorrect_prior.png}
\end{center}
\end{minipage}
\vspace{.3in}
\caption{Experiment on multivariate Gaussian distribution.}
\label{fig:mv_incorrect}
\end{figure}
\end{comment}

\subsection{Ecological dynamic systems: blowfly}
\label{sec:blowfly}
Following \citet{pmlr-v51-park16}, we performed an experiment on parameter estimation with a dynamical system of blowfly populations \citep{Woo10}, which is defined as
\begin{equation*}
N_{t+1} = PN_{t-\tau} \exp\left(- N_{t-\tau} / N_{0} \right)\it{e_{t}} + \it{N_{t}} \exp(-\delta \epsilon_{t}),
\end{equation*}
where $t = 1,\dots,T$ are time indices, $N_t$ is the population at time $t$, $e_{t} \sim {\rm Gam}(\frac{1}{\sigma_{p}^{2}} \sigma_{p}^{2})$ and $\epsilon_{t} \sim {\rm Gam}(\frac{1}{\sigma_{d}^{2}}, \sigma_{d}^{2})$ are independent Gamma-distributed noise, and $\theta := (P \in \mathbb{N},\ N_{0} \in \mathbb{N},\ \sigma_{d} \in \mathbb{R}_{+},\ \sigma_{p} \in \mathbb{R}_{+},\ \tau \in \mathbb{N},\ \delta\ \in \mathbb{R}_{+})$ are the parameters of the system.
%\footnote{$\mathbb{R}_{+}$ denotes a set of positive real numbers.}
%and $t > \tau$.
The task is to estimate $\theta$ from observed values of $N_1,\dots,N_T$.
We set the true parameters as $\theta = (29, 260, 0.6, 0.3, 7, 0.2)$, and the time-length $T$ for both the observed and pseudo data as $T = 1000$. Following \citet[Sec.~4]{pmlr-v51-park16}, for each parameter we defined a Gaussian prior on its logarithm (see the Supplementary Materials for a definition). In this experiment, for all methods we converted the observed and pseudo-data into histograms with 1000 bins (i.e., we treated each data as a 1000 dim.~vector), as this produced better results. 
K2-ABC and K-ABC used 1300 pairs of a parameter and a pseudo-data item.
For the proposed method and SMC-ABC, we generated 100 pairs of a parameter and pseudo-data for the initial iteration, and the iterations were then repeated 13 times, resulting in total 1300 simulations. 
For SMC-ABC, we set the parameter $\alpha \in (0,1)$ to be 0.3, as in Sec.~\ref{sec:Gauss-misspecified}.

For each method we performed 30 independent trials, and the results are summarized in Table \ref{tab:table1}. 
%Parameter error is the mean absolute value of each estimated parameter subtracted by the true parameter and divided by the true parameter. 
%Data error is the energy distance between the data generated using the true and estimated parameters. 
%The results are averages and standard deviations. 
The proposed method performed the best, even when the number of simulations was halved (i.e., KR-ABC (less)). %\footnote{
%Because of computational resource limitations, we were not able to conduct experiments on the same PCs across all methods, and, therefore, CPU-time is not discussed here.}
%The priors and true parameters for the population dynamics were from the code of k2-ABC, but they were not explicitly written anywhere.
%

\begin{table}[t]
  \caption{Results for blowfly population dynamics in Sec.\ref{sec:blowfly}}
    \begin{center}
    \begin{adjustbox}{max width=210pt}
  \begin{tabular}{|l|c|r|r|r|} \hline
    Algorithm & parameter error & data error & cputime\\ \hline 
    KR-ABC & 0.47(0.11) & 43.85(37.24) &  101.143(13.25) \\ \hline
    KR-ABC (less) & 0.57(0.21) & 67.57(47.11) &  32.98(1.21) \\ \hline
    K2-ABC & 0.81(0.42) & 67.45(77.86) &  23.47(1.59) \\ \hline
    K-ABC & 0.62(0.09) & 89.37(29.22) &  30.66(2.57)  \\ \hline
    SMC-ABC (mean) & 0.83 (0.10) & 170.41 (47.91) &  38.50(2.34) \\ \hline
	SMC-ABC (MAP) & 0.84 (0.12) & 163.19(42.51) &  38.50(2.34) \\ \hline
    ABC-DC & 0.89(0.17) & 134.12(58.92) &  29.94(4.57)\\ \hline
    BO &0.70(0.28) & 108.18(67.08) & 3217.40(157.31)\\ \hline
    MSM & 0.67(0.08) & 89.17(33.20) &  25.46(8.26)\\ \hline
  \end{tabular}
  \end{adjustbox}
  \label{tab:table1}
  \end{center}
  \vspace{-0.15in}
\end{table}

\subsection{Multivariate elliptically contoured alpha stable distribution}
\label{sec:alpha-stable}
In addition to the competitive methods described earlier, we performed a comparison with the method called {\em noisy ABC-MLE} \citep{yildirim2015parameter}.
%This is a gradient-based method for MLE with intractable likelihood. 
This method assumes that sampling from the intractable model can be realized by deterministic mapping applied to a simple random variable, and that the gradient of the deterministic mapping is available. 
This method can, then, only be applied to a limited class of generative models, though for such models it can perform well.
Although the main scope of this paper is on simulation models in which such gradient information is unavailable, we performed this experiment in order to see how the proposed method compared with this method without relying on the gradient information.

%In particular, here we tested the stability of kernel recursive ABC for high-dimensional parameters. 

We also considered parameter estimation with {\em multivariate elliptically contoured alpha stable distributions} \citep{nolan2013multivariate}, which subsume heavy-tailed and skewed distributions and are popular for modeling financial data.
This family of distributions in general does not admit closed-form expressions for density functions, which means they are ``intractable'' in the sense that the standard procedure for parameter estimation cannot be employed.
However, sampling of a random vector ${\bm X} \in \R^d$ from this family is possible in the following way:
\begin{eqnarray*}
{\bm X}& :=& A^{1/2} {\bm G} + \delta \in \R^d,\quad 
{\bm G} \sim  {\rm Normal}({\bm 0},Q), \\
%A& \sim & {\bm S}(\alpha/2, 1, 2(\cos \pi\alpha/4)^{2/\alpha},0),\\
A& :=& \tau_{\theta}(U_1,U_2) \in \R,\\
%U & := & (U_1, U_2),\\
&& U_1  \sim   {\rm Unif}(-\pi/2, \pi/2),\quad U_{2} \sim {\rm Exp}(1),
\end{eqnarray*}
where  $Q \in \mathbb{R}^{d \times d}$ is positive definite,  $\delta \in \mathbb{R}^{d}$, $\theta := (\alpha, \beta, \mu, \sigma) \in (0,2] \times [-1,1] \times \R \times [0,\infty)$,  $\tau_{\theta}$ is a deterministic mapping whose concrete form is described in the Supplementary Materials, and ${\rm Unif}$ and ${\rm Exp}$ denote uniform and exponential distributions, respectively.

We dealt with estimation of $\alpha := 1.3$ and $Q$, while fixing the other parameters as $\delta := {\bm 0}$, $\beta := 1$, $\mu := 0$, and $\sigma := 1$.
We restricted $Q$ to be a positive definite matrix such that all diagonal elements are the same and so are the off-diagonal elements.
We defined true $Q$ to be a matrix whose diagonal elements are $1.0$ and off-diagonals are $0.2$.
Therefore the task was to estimate these three values (i.e., $1.3$ for $\alpha$, and $1.0$ and $0.2$ for $Q$).
We used ${\rm Unif}[0,2]$ as a prior for $\alpha$, and ${\rm Unif}[0,5]$ as a prior for each of the diagonal and off-diagonal values of $Q$.
 
%where??? ${\bm S}$ is the univariate alpha stable distribution \citep{chambers1976method}. It also lacks a closed-form expression of the density. The important point is that the random samples of the multivariate alpha stable distribution are obtained through the deterministic mapping $\tau_{\theta}$ applied to the random samples originating from uniform distribution Unif and exponential distribution Exp. 
%Utilizing this characteristic, noisy ABC-MLE uses the gradient of the deterministic mapping to obtain the MLE of the multivariate stable alpha distribution \citep{yildirim2015parameter}; that is, it utilizes more information than kernel recursive ABC. 

Figure \ref{fig:alpha} shows results for the averages of mean square errors in parameter estimation over 30 independent trials, with variation in the dimensionality $d$ from $2$ to $16$.
For each method we sampled a total of 1400 pairs of a parameter and a pseudo-data item, and for iterative methods we used 100 pairs in each iteration.
Each pseudo-data (and observed-data) was made up of 1000 points.
The noisy ABC-MLE exploited the gradient information in $\tau_\theta$, while the other methods did not.
The proposed method was competitive with BO and outperformed the other methods with the exception of the noisy ABC-MLE.
Although the noisy ABC-MLE was accurate for lower-dimensionality (as expected), it exhibited a steep increase in errors for higher dimensionality. In contrast, the performance degradation of the proposed method was mild for higher dimensionality.

%the??? noisy ABC-MLE was accurate for lower dimensionality, and its performance degraded as dimensionality increased.

%that??? although the proposed method underperformed in low-dimensional settings, it was more robust in high-dimensional settings and eventually outperformed noisy ABC-MLE. 
%This is remarkable, given that the proposed method runs on the basis of less information.
%
\begin{figure}[t]
%\vspace{-0.1in}
\begin{minipage}{0.48\hsize}
\begin{center}
\includegraphics[scale = 0.12]
%{alpha_stable_result_all_methods.png}
{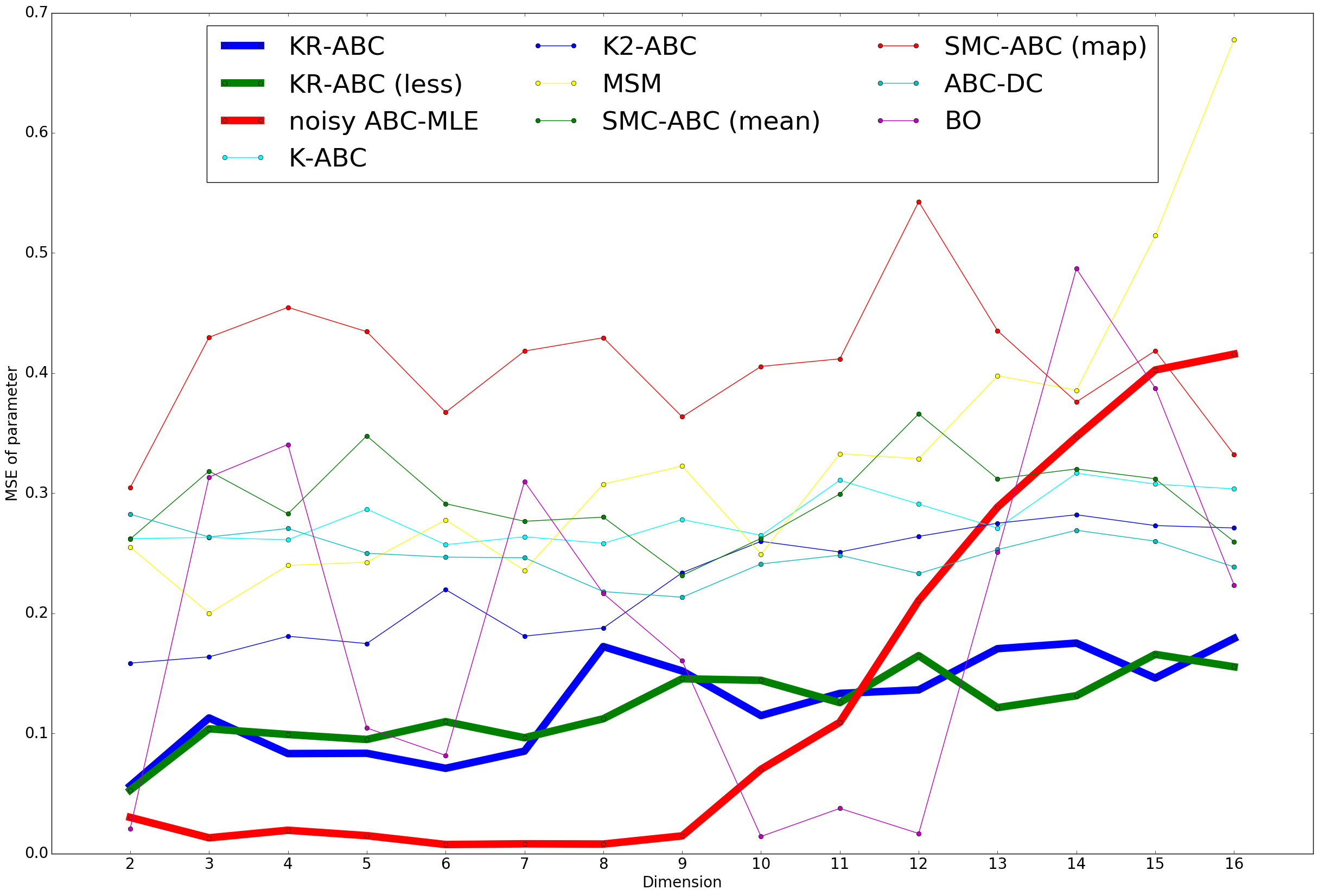}
\end{center}
\end{minipage}
\caption{Results for multivariate elliptically contoured alpha stable distributions in Sec.~\ref{sec:alpha-stable}. 
%Results for the methods of interest are shown in thick lines. 
For each dimension (vertical axis), the averages of mean squared errors (MSE) over 30 trials are shown.}
\vspace{-0.15in}
\label{fig:alpha}
\end{figure}

\begin{comment}
The difficult part to deal with is the unavailability of closed-form expressions of the density. Rather, its characteristic function is given by
%
\begin{eqnarray*}
\mathbb{E}[\exp(\sqrt{-1}{\bf u}^T {\bf X})] = \exp (-({\bf u}^{T} Q
{\bf u})^{\alpha/2} + \sqrt{-1} {\bf u}^{T}\delta).
\end{eqnarray*}

for a positive semi-definite matrix $Q \in \mathbb{R}^{d \times d} $ and shift vector $\delta \in \mathbb{R}^{d}$. 
\end{comment}

\subsection{Gaussian mixture with redundant components}
\label{sec:gaussian-mixture}
We consider here a parametric model in which there exist redundant parameters for expressing given data.
We are interested in whether point estimation with the proposed method results in elimination of the redundant parameters when applied to such a model.
This was motivated by \citet{Yamazaki2013}, who argued that, for mixture models, the use of a Dirichlet prior with a sufficiently small concentration parameter leads to elimination of unnecessary components.
We therefore focus on mixture models with redundant components.
%Although its likelihood is tractable, we treat it as intractable in order to evaluate the effectiveness of the proposed method. 

Specifically, we considered Gaussian mixture models.
%\footnote{Although the likelihood function of a Gaussian mixture is not intractable, we used them for the sake of simplicity.}
We defined the true model as a {\em two-component} Gaussian mixture
$\sum_{i=1}^{2} \phi_i \rm{Normal}(\mu_i, 20)$ of equal variances. 
The task was to estimate the mixture coefficients $(\phi_1,\phi_2,) := (0.7, 0.3)$ and the associate means $(\mu_1,\mu_2) := (110, 70)$, provided 3000 i.i.d.~sample points from the model as observed data $y^*$.
We employed an over-parametrized model for point estimation (i.e., no method used the knowledge that the truth consisted of 2 components), which is a {\em four-component} Gaussian mixture $\sum_{i=1}^{4} \phi_i \rm{Normal}(\mu_i, 20)$. 
We used a 4-dimensional Dirichlet distribution with equal concentration parameters $0.01$ as a prior for the coefficients $(\phi_1,\dots,\phi_4)$, and ${\rm Normal} (0,100)$ as a prior for each of $\mu_1,\dots,\mu_4$.

For each method, we generated a total of 1000 pairs of a parameter and pseudo-data, and for iterative methods, we made use of 100 pairs in each iteration, resulting in 10 iterations.
Each pseudo-data consisted of 3000 simulated observations.
For all the methods, we converted each data item into a histogram of 300 bins and treated it as a 300 dim.~vector since this resulted in better performances. 
We set the parameter $\alpha \in (0,1)$ of SMC-ABC to be 0.2, as this performed well in this experiment.

%For the proposed method and K-ABC, the simulated data were converted into histograms with 300 bins ranging from 0 to 300 because that produced better results. 
%The selection of the associated parameters and the construction of point estimates for each method was the same as that described earlier. 

We ran each algorithm 30 times, and the resulting average errors and standard deviations are shown in Table \ref{tab:table2}, where the $\phi$ error and $\mu$ error denote the errors for the coefficients and the means, respectively, as measured in terms of Euclidean distance. 
More precisely, since any permutation of component labels will result in the same model, we first sorted the estimated parameters $\{(\phi_i,\mu_i)\}$ so that $\phi_1 \geq \cdots \geq \phi_4$, and we then measured the errors w.r.t.~the ground truth $\phi := (0.7,0.3,0,0)$ and $\mu := (110,70)$.
For the $\mu$ error, we computed the errors only for the estimated means $\mu_1, \mu_2$ associated with the two largest coefficients since there was no ground truth for the redundant components $\mu_3, \mu_4$.
Results show that the proposed KR-ABC performed best, indicating that the dominant components were successfully estimated.

\begin{table}[t]
  \caption{Results for the Gaussian mixture model in Sec.~\ref{sec:gaussian-mixture}} \label{tab:table2}
   \begin{center}
   \begin{adjustbox}{max width=210pt,max height=40pt}
   \begin{tabular}{|l|c|r|r|r|r|} \hline
    Algorithm & $\phi$ error & $\mu$ error & data error & cputime \\ \hline
    KR-ABC & 0.159(0.106) & 54.14(8.71) & 0.03(0.05) & 281.59(12.55)\\ \hline
    KR-ABC (less) &0.22(0.13) & 64.04(17.58) & 0.11(0.15) & 147.39(16.76) \\ \hline
    K2-ABC & 0.53(0.02) & 93.98(3.84) & 0.69(0.10)& 89.43(7.56)\\ \hline
    K-ABC & 0.50(0.02) & 92.57(12.77) & 0.67(0.17)& 135.01(11.35)\\ \hline
    SMC-ABC (mean) & 0.51(0.04) & 83.19(27.86) & 0.23(0.15)& 214.35(8.77) \\ \hline 
    SMC-ABC (MAP) & 0.21(0.13) & 72.76(56.24) & 0.12(0.08)& 214.35(8.77) \\ \hline 
    ABC-DC & 0.48(0.16) & 137.84(50.06) & 0.43(0.14)& 149.74(21.22)\\ \hline
    BO & 0.37(0.08) & 80.79(36.17) & 0.82(0.68)& 13775.48(1438.9)\\ \hline
    MSM & 0.24(0.08) & 117.02(11.48) & 0.24(0.08)& 171.58(59.38)\\ \hline
  \end{tabular}
  \end{adjustbox}
  \end{center}
\vspace{-0.25in}
\end{table}

\subsection{Real-world pedestrian simulator}
\label{sec:pedestrian-sim}
Our final experiment was parameter estimation with {\em CrowdWalk}, a publicly available real-world simulator\footnote{https://github.com/crest-cassia/CrowdWalk} for the movements of pedestrians in a commercial district \citep{Noda2010p}.
It has been used to gain insights into pedestrian behavior at a variety of events and occurrences, such as fireworks festivals and evacuations after earthquakes. 
As this simulator is complicated and also computationally expensive, its likelihood function is intractable.

%\footnote{There is a practical motivation for modeling pedestrians as a mixture model, as it makes it easier to interpret their behavior.} 

Using CrowdWalk, we simulated the movements of pedestrians in Ginza, a commercial district in Tokyo (see Supplementary Materials for an illustration).
Specifically, we modeled pedestrians as a mixture of multiple groups,
each of which has the following 6 parameters (below $i$ denotes the index of a group):
(1) $\theta^{(N)}_i \in \mathbb{N}$: the number of pedestrians in the group; 
(2) $\theta^{(T)}_i \in \R_{+}$: the time when the group starts to move; 
(3) $\theta^{(S)}_i \in \R^2$: the starting location of the group (e.g., stations); 
(4) $\theta^{(G)}_i \in \R^2$: the goal location of the group; 
(5) $\theta^{(P)}_i \in \R^2$: the intermediate location(s) that the pedestrians in the group visit (e.g., stores); and 
(6) $\theta^{(R)}_i \in \R_{+}$: the time duration(s) of the pedestrians' visit(s) at the intermediate location(s).

In this experiment, we focused on estimation of the first two parameters $\theta^{(N)}_i$, $\theta^{(T)}_i$, and fixed the other parameters.
We defined the true model as a mixture of 5 pedestrian groups, and set their parameters as $(\theta^{*(N)}_1,\dots,\theta^{*(N)}_5) := (100,100,100,100,100)$ and $(\theta^{*(T)}_1,\dots,\theta^{*(T)}_5) := (30, 60, 90, 120, 150)$. 
As in Sec.~\ref{sec:gaussian-mixture}, we used a redundant model of a mixture of 10 groups for parameter estimation.
The goal was to detect the active 5 groups of the true model, without knowing that the truth consists of 5 groups.
For simplicity, 5 (unknown) groups among the 10 candidate groups included the parameters of the true model other than $\theta^{*(N)}_i$, $\theta^{*(T)}_i$; see the Supplementary Materials for details.

We defined prior distributions as follows.
First we assumed the total number $500$ of pedestrians to be known.
%\footnote{This can be justified in practice, as an estimate of the number of people visiting a commercial district is often available.}
The mixing coefficients of the mixture of 10 groups are given by $(\phi_1,\dots,\phi_{10}) = (\theta^{(N)}_1,\dots,\theta^{(N)}_{10}) / 500$.
Thus, rather than directly putting a prior on $(\theta^{(N)}_1,\dots,\theta^{(N)}_{10})$, we defined a prior on the mixing coefficients $(\phi_1,\dots,\phi_{10})$.
Specifically, we used a Dirichlet prior with a small concentration parameter, as in Sec.~\ref{sec:gaussian-mixture}, in order to eliminate 5 redundant components:
\begin{eqnarray*}
(\phi_1,\dots,\phi_{10}) &\sim& {\rm Dirichlet}(\alpha_{1} ,..., \alpha_{10} ), \\
\theta^{(N)}_{i} &:=&  \phi_i * 500, \quad (i = 1,\dots, 10)
%\sum^{10}_i \phi_i &=& 1, ~~~ \phi_i \geq 0,
\end{eqnarray*}
where $\alpha_1 = \cdots = \alpha_{10} = 0.01$ denote the concentration parameters.
For each of $\theta^{(T)}_1,\dots,\theta^{(T)}_{10}$, we defined a broad uniform prior 
$\theta^{(T)}_i \sim {\rm Unif}(0,480)$.

From the true model, we simulated 4200 time steps of pedestrian flow as observed data.
We made $5 \times 5 = 25$ grids in a map of Ginza and computed a histogram of the corresponding 25 bins for each time step.
Thus, observed data was made up of 4200 vectors in $\R^{25}$.
%The construction of the observed and pseudo-data was the same as in the previous experiment, except that the number of grids was 25 ($5 \times 5$) 
%this??? time in order to obtain higher resolution in the data. 
In the same way, each method generated a total of 4200 vectors, and each iterative method made use of 200 vectors in each iteration, running 21 iterations in total.
For SMC-ABC, we set the parameter $\alpha \in (0,1)$ to be 0.2, as in the previous experiment. 

%200??? samples from the prior distribution, and continued to run until the total number of simulation steps reached 4200. The parameters for each algorithm were selected in the same way as was described earlier. 

We ran each method 20 times, and the resulting averages and standard deviations for errors are summarized in Table \ref{tab:table3}, where ``$\theta^{(N)}$ error'' and ``$\theta^{(T)}$ error'' denote the errors of the corresponding estimated parameters, as measured in terms of Euclidean distance.
These errors were computed in the same way as in Sec.~\ref{sec:gaussian-mixture} (e.g., the estimated parameters were sorted according to the magnitudes of the mixing coefficients).
%”Data error” denotes the energy distance between the observed data and that simulated with the estimated parameters.
Results show that our method performed the best, confirming its effectiveness.
In the Supplementary Materials, we also report the point estimates made using the proposed method, showing that the true parameters were estimated reasonably accurately.
\begin{table}[t]
  \caption{Results for the pedestrian simulator in Sec.~\ref{sec:pedestrian-sim}} \label{tab:table3}
  \begin{center}
  \begin{adjustbox}{max width=220pt}
  \begin{tabular}{|l|c|r|r|r|r|} \hline
    Algorithm & $\theta^{(N)}$ error & $\theta^{(T)}$ error & data error& cputime \\ \hline
    KR-ABC & 61.58(74.42) & 70.93(102.08) & 0.008(0.009) & 2233.45(97.54)\\ \hline
    KR-ABC (less) & 82.46(75.05) & 134.00(161.85) & 0.014(0.014) & 1875.32(147.16)\\ \hline
    K2-ABC & 298.94(120.71) & 308.95(109.43) & 0.10(0.10)& 1547.32(56.31)\\ \hline
    K-ABC & 354.72(145.76) & 389.52(140.91) & 0.12(0.09)& 1773.74(84.91)\\ \hline
    SMC-ABC (mean) & 271.51(104.64) & 363.12(91.28) & 0.09(0.07)& 2017.89(110.02)\\ \hline
    SMC-ABC (MAP) & 255.15(139.33) & 348.43(104.74) & 0.09(0.1)& 2017.89(110.02)\\ \hline
    ABC-DC & 273.93(136.14)& 327.48(98.12) & 0.09(0.14) & 1984.43(59.12)\\ \hline
    BO & 194.57(65.83) & 291.73(105.33) & 0.04(0.06) & 37541.23(3047.46)\\ \hline
    MSM & 453.58(89.43) & 510.04(55.10) & 0.24(0.17) & 1869.83(49.51)\\ \hline

  \end{tabular}
  \end{adjustbox}
  \end{center}
  \vspace{-0.25in}
\end{table}

\section{Summary and future work} \label{sec:conclusions}
We have proposed kernel recursive ABC for point estimation with intractable likelihood and have empirically investigated the effectiveness of this approach. While we have also provided theoretical analysis to a certain extent, there remain important theoretical topics, as discussed in Sec.~\ref{sec:theory}, that we wish to reserve for future research.

%For point estimation, we have favored ML over maximum a posterior (MAP) estimate. MAP may be unreliable, if the prior distribution cannot be specified appropriately. Moreover, as we show empirically in Sec. 4.2, existing approaches to point estimation with intractable likelihood also suffer from the misspecification of a prior distribution. Therefore it is desirable to develop a method for MLE that does not depend on the choice of a prior, or that is robust to the misspecification of a prior%

%We hope that our method will provide a useful alternative to ABC methods in situations in which point estimation is desirable, such as when one lacks appropriate prior knowledge w.r.t. model parameters, or when one performs prediction with a computationally-demanding simulation model, in which cases Bayesian approaches may not be the best choice.

\subsection*{Acknowledgements}
We thank the anonymous reviewers as well as Itaru Nishioka, Itsuki Noda, Takashi Washio, Shinji Ito, Wittawat Jitkrittum, and Marie Oshima for their helpful comments and support. We also thank Shuhei Mano for providing his code.
MK has been supported by the European Research Council (StG Project PANAMA).
KF has been supported by JSPS KAKENHI 26280009.

\bibliographystyle{natbib}
\bibliography{krabc}

\newpage
\appendix
\onecolumn

\
%\vspace{-0.1in}
\begin{center}
   \section*{\centering Supplementary Materials}
\end{center}
\vspace{0.3in}

%\title{Supplementary Materials}
%\date{}   
%\author{\empty}
%\maketitle

\setcounter{assumption}{0}
\setcounter{proposition}{0}

\section{Proofs for theoretical results}
We here provide proofs for the theoretical results in Section 3 of the main text.
For ease of understanding, we repeat the assumptions and the statements w.r.t. those results.
The notation follows that of the main text.

\begin{assumption} \label{as:likelihood_prior}
(i) $\ell$ has a unique global maximum at $\theta_\infty \in \Theta$, and $\pi(\theta_\infty) > 0$;
(ii) $\pi$ is continuous at $\theta_\infty$, $\ell$ has continuous second derivatives in the neighborhood of $\theta_\infty$, and the Hessian of $\ell$ at $\theta_\infty$ is strictly negative-definite.
\end{assumption}

\begin{proposition} \label{prop:convergence}
Let $\Theta \subset \R^d$ be a Borel measurable set, $k: \Theta \times \Theta \to \R$ be a continuous, bounded kernel, and $\H$ be its RKHS.
If Assumption \ref{as:likelihood_prior} holds, then we have 
\begin{eqnarray*}
\lim_{N \to \infty} \left\| \mu_{P_N} - k(\cdot,\theta_\infty) \right\|_{\mathcal{H}} = 0.
\end{eqnarray*}
\end{proposition}
\begin{proof}
Because Assumption \ref{as:likelihood_prior} is equivalent to Assumptions A1, A2 and A3 in \citet{LelNadSch10}, we can use the Corollary to Lemma A.2 on p.1624 of \citet{LelNadSch10}; this guarantees the weak convergence of $P_N$ to $\delta_{\theta_\infty}$, the Dirac distribution at $\theta_\infty$.
Therefore, 
\begin{eqnarray}
 \lim_{N \to \infty} \left\| \mu_{P_N} - k(\cdot,\theta_\infty) \right\|_{\mathcal{H}}^2 \nonumber 
&=& \lim_{N \to \infty} \left< \mu_{P_N}, \mu_{P_N} \right>_{\mathcal{H}} - 2 \lim_{N \to \infty} \left< \mu_{P_N}, k(\cdot,\theta_\infty) \right>_{\mathcal{H}} \nonumber \\ && +  \left<k(\cdot,\theta_\infty), k(\cdot,\theta_\infty) \right>_{\mathcal{H}} \nonumber \\
&=& \lim_{N \to \infty} \int \int k(\theta,\theta') dP_N(\theta) dP_N(\theta') \nonumber \\
&& - 2 \lim_{N \to \infty} \int k(\theta,\theta_\infty)dP_N(\theta) + k(\theta_\infty,\theta_\infty) \nonumber \\
&=& k(\theta_\infty,\theta_\infty) - 2 k(\theta_\infty,\theta_\infty) + k(\theta_\infty,\theta_\infty) \label{eq:limit} \\
&=& 0, \nonumber
\end{eqnarray}
where (\ref{eq:limit}) follows from the weak convergence of $P_N$ to $\delta_{\theta_\infty}$ and $k$ is continuous and bounded. Here we have used Theorem 2.8 (ii) in \cite{Bil99} for the first term in (\ref{eq:limit}).
\end{proof}

\begin{assumption} \label{as:kernel}
(i) There exists a constant $C > 0$ such that $k(\theta,\theta) = C$ for all $\theta \in \Theta$.
(ii)  It holds that $k(\theta,\theta') < C$ for all $\theta, \theta' \in \Theta$ with $\theta \neq \theta'$.
\end{assumption}

\begin{comment}
\begin{corollary}
\label{cor:consistency}
\begin{eqnarray}
\argmin_{x}  \lim_{N \to \infty} \left\lVert \mu_{p_N} - k(\cdot,x)  \right\rVert_{\mathcal{H}}^2 = \theta_{\infty}.
\end{eqnarray}

\begin{proof}
\begin{eqnarray}
&& \argmin_{x}  \lim_{N \to \infty} \left\lVert \mu_{p_N} - k(\cdot,x)  \right\rVert_{\mathcal{H}}^2 \\
&=& \argmin_{x} \lim_{N \to \infty} -2 \int k(x,\theta)dP_{N}(\theta) \\
&=& \argmax_{x} \int k(x,\theta)d\delta_{\theta_{\infty}}(\theta)\\
&=& \argmax_{x} k(x,\theta_{\infty})\\
&=& \theta_{\infty}
\end{eqnarray}
\end{proof}
\end{corollary}
\begin{remark}
Corollary \ref{cor:consistency} states that the first sample from the power posterior $P_N$ via kernel herding is consistent for MLE as $N \to \infty$.
\end{remark}
\end{comment}

\begin{proposition}
\label{pr:consistency2}
Let $\Theta \subset \R^d$ be a compact set and $k: \Theta \times \Theta \to \R$ be a continuous, bounded kernel. 
Let $\theta_N := \argmin_{\tilde{\theta} \in \Theta} \left\lVert \mu_{P_N} - k(\cdot,\tilde{\theta}) \right\rVert_{\mathcal{H}}$.
If Assumptions \ref{as:likelihood_prior} and \ref{as:kernel} hold, then we have $\theta_N \to \theta_{\infty}$ as $N \to \infty$.
\end{proposition}

\begin{proof}
By the reproducing property and Assumption \ref{as:kernel}, we have
\begin{eqnarray}
\| \mu_{P_N} - k(\cdot,\tilde{\theta})\|_{\H}^2 
&=& \left\| \mu_{P_N} \right\|_{\mathcal{H}}^2 \nonumber -2 \mu_{P_N}(\tilde{\theta}) + k(\tilde{\theta}, \tilde{\theta}) \\
&=&  \left\| \mu_{P_N} \right\|_{\mathcal{H}}^2 \nonumber -2 \int k(\tilde{\theta},\theta)dP_N(\theta) + C.
\end{eqnarray}
Since $\int k(\tilde{\theta},\theta)dP_N(\theta)$ is a continuous function of $\tilde{\theta}$ (which follows from the continuity of $k$ and the dominated convergence theorem), it follows that $\| \mu_{P_N} - k(\cdot,\tilde{\theta})\|_{\H}^2$ is a continuous function of $\tilde{\theta}$, and so is $\| \mu_{P_N} - k(\cdot,\tilde{\theta})\|_{\H}$.
Thus, since $\Theta$ is compact,  $\theta_N =  \argmin_{\tilde{\theta} \in \Theta} \left\lVert \mu_{P_N} - k(\cdot,\tilde{\theta}) \right\rVert_{\mathcal{H}}$ exists.
Using the above identity, we then have
\begin{eqnarray*}
\theta_N &=& \argmin_{\tilde{\theta} \in \Theta} \left\lVert \mu_{P_N} - k(\cdot,\tilde{\theta})  \right\rVert_{\mathcal{H}}^2 \\
&=& \argmin_{\tilde{\theta} \in \Theta} \left\| \mu_{P_N} \right\|_{\mathcal{H}}^2 \nonumber -2 \mu_{P_N}(\tilde{\theta}) + C \\
&=& \argmax_{\tilde{\theta} \in \Theta} \mu_{P_N}(\tilde{\theta}).
\end{eqnarray*}
By the reproducing property, the Cauchy-Schwartz inequality, and Assumption \ref{as:kernel}, we have for all $\theta \in \Theta$
\begin{eqnarray}
\left| \mu_{P_N}(\theta) -  k(\theta,\theta_\infty) \right| 
&=& \left| \left< k(\cdot,\theta), \mu_{P_N} -  k(\cdot,\theta_\infty) \right> \right| \nonumber \\
&\leq& \sqrt{ k(\theta,\theta) } \left\| \mu_{P_N} -  k(\cdot,\theta_\infty) \right\|_{\mathcal{H}} \nonumber \\
&=& \sqrt{ C } \left\| \mu_{P_N} -  k(\cdot,\theta_\infty) \right\|_{\mathcal{H}} \label{eq:uniform}
\end{eqnarray}

Let $\varepsilon$ be an arbitrary positive number and $U_\varepsilon(\theta_\infty)$ be an open $\varepsilon$-neighborhood of $\theta_\infty$.  From Assumption \ref{as:kernel} (ii) and the continuity of $k$, there is $\delta>0$ such that 
\begin{equation}
\max_{\theta\in\Theta\backslash U_\varepsilon(\theta_\infty)} k(\theta,\theta_\infty) \leq C-\delta. \label{eq:outof}
\end{equation}
It follows from Eq.(\ref{eq:uniform}) and Proposition \ref{prop:convergence} that there is $N_0\in\mathbb{N}$ such that 
\begin{equation}
\max_{\theta\in\Theta} \left| \mu_{P_N}(\theta) -  k(\theta,\theta_\infty) \right| \leq \delta/3 \label{eq:unifbound}
\end{equation}
holds for all $N\geq N_0$.
This implies, in particular, that for all $N\geq N_0$ 
\begin{equation}
\mu_{P_N}(\theta_\infty) \geq k(\theta_\infty,\theta_\infty) -\delta/3 = C-\delta/3.  \label{eq:lbpn}
\end{equation}
On the other hand, using Eqs.(\ref{eq:outof}) and (\ref{eq:unifbound}), we have
\begin{equation}
\max_{\theta\in\Theta\backslash U_\varepsilon(\theta_\infty)} \mu_{P_N}(\theta)  
\leq C-\frac{2}{3}\delta \label{eq:outof2}
\end{equation}
for all $N\geq N_0$.  

Eqs.(\ref{eq:lbpn}) and (\ref{eq:outof2}) show that the maximum of $\mu_{P_N}$ is attained in $U_\varepsilon(\theta_\infty)$, that is, $\theta_N \in U_\varepsilon(\theta_\infty)$, for all $N\geq N_0$, which completes the proof. 
\end{proof}

\begin{remark}
For simplicity, we assume in Proposition \ref{pr:consistency2} that $\theta$ is compact, but this condition can be relaxed.
For example, we may instead assume the following weaker condition:
For any open neighborhood $U$ of $\theta_\infty$, there is a positive constant $\delta$ such that
$\sup_{\theta\in \Theta\backslash U} k(\theta,\theta_\infty) \leq k(\theta_\infty,\theta_\infty) - \delta$.
\end{remark}

\section{Demonstration of the auto-correction mechanism for a misspecified prior}

\begin{figure}[H]
\vspace{.3in}
\begin{minipage}{1.0\hsize}
\begin{center}
\includegraphics[scale = 0.5]{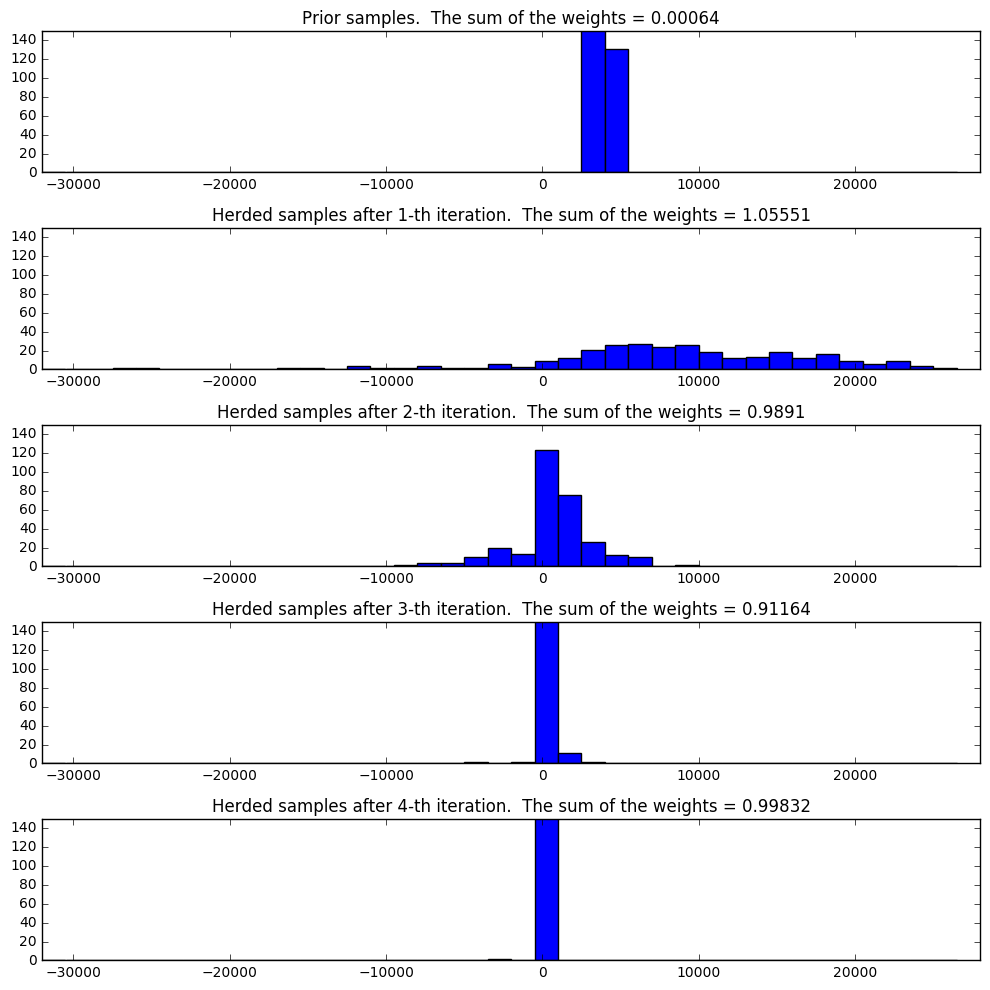}
\end{center}
\end{minipage}
\vspace{.3in}
\caption{Each figure shows a histogram of simulated parameters for the mean of the Gaussian distribution in each iteration, as produced with the proposed method. ``The sum of the weights'' on the top of each figure is the sum of the weights given by kernel ABC at each iteration, as defined by Eq.~(3) of the main text.}
\label{fig:mv_incorrect}
\end{figure}

We demonstrate here how the auto-correction mechanism of the proposed method works; for an explanation of this mechanism, see Section 3 of the main text.
We performed an experiment similar to the one in Section 4.2 of the main text, but under a simpler setting.
The task was to estimate the mean $0$ of a univariate Gaussian distribution, ${\rm Normal}(0,40)$, provided 100 i.i.d.~observations from it.
The variance $40$ was assumed to be known.
For the prior distribution over the mean, we used the uniform distribution on $[2000,3000]$, which is severely misspecified.
For the proposed method, we recomputed the bandwidth of a Gaussian kernel for each iteration, by using the median heuristic with simulated data.
In each iteration, 300 pseudo-observations were generated for the proposed method.

Figure \ref{fig:mv_incorrect} shows the results for the first 4 iterations.
The top figure is a histogram of the parameters generated from the prior distribution, which, because of the misspecification of the prior, do not cover the true mean $0$.
The resulting sum of the weights is $0.00064$, implying that the simulated pseudo-observations are far apart from the observed data.
(As explained in the caption of Figure \ref{fig:mv_incorrect}, ``The sum of the weights'' on the top of each figure is the sum of the weights $w_1,\dots,w_n$ given by kernel ABC at each iteration, as defined by Eq.~(3) of the main text.)
The second figure is a histogram of the parameters generated by kernel herding in the first iteration.
These parameters were generated so as to explore the parameter space, in response to the auto-correction mechanism explained in Section 3 of the main text. 
Since the simulated parameters were now scattered around the true mean 0, kernel ABC began to perform well from the next iteration.
After only 4 iterations, the simulated parameters concentrated around the true mean.

%For the first iteration, the proposed method failed to construct the kernel mean of the posterior distribution properly (Figure top).  This is shown by the fact that the sum of the weights was far less than 1. 

\section{Supplementary to the population dynamics experiment in Section 4.3}

We offer here supplementary materials for the experiment on the blowfly population dynamics in Section 4.3 of the main text.

\subsection{Errors for individual parameters}

\begin{table}[H]
   \caption{Results for blowfly population dynamics in Sec. 4.3}
    \begin{center}
    \begin{adjustbox}{max width=430pt}
  \begin{tabular}{|l|c|r|r|r|r|r|r|r|r|} \hline
    Algorithm & $P$ & $N_{0}$ & $\sigma_{d}$ & $\sigma_{p}$ & $\tau$ & $\delta$ & data & cputime\\ \hline 
    KR-ABC & 0.28(0.13) & 0.03(0.05) & 0.93(0.55) & 1.22(0.64) & 0.17(0.14) & 0.17(0.15) & 43.85(37.24)  &101.143(13.25) \\ \hline
    KR-ABC (less) & 0.15(0.13) & 0.10(0.07) &  1.11(0.12) &
    1.45(0.26) & 0.23(0.41) & 0.27(0.45) & 67.57(47.11) & 32.98(1.21) \\ \hline
    K2-ABC & 1.27(1.75) & 0.20(0.23) & 0.98(0.42) & 1.46(1.10) &
       0.31(0.19) & 0.61(0.82) & 67.45(77.86) &  23.47(1.59) \\ \hline
    K-ABC & 0.48(0.13) & 0.14(0.06) & 1.28(0.87) & 1.42 (0.40) & 0.22 (0.02) &
       0.27(0.25) & 89.37(29.22) &  30.66(2.57)  \\ \hline
    SMC-ABC (mean) & 0.58 (0.15) & 0.11(0.06) & 1.03(0.46) & 1.98(0.32) & 0.28(0.15) &
        1.01(0.11) & 170.41 (47.91) &  38.50(2.34) \\ \hline
	SMC-ABC (MAP) & 0.51(0.28) & 0.19(0.10) & 0.89(0.33) & 1.89(0.33) & 0.53(0.46) &
       1.01(0.10) & 163.19(42.51) &  38.50(2.34) \\ \hline
    ABC-DC & 0.48(0.22) & 0.25(0.13) & 1.36(0.88) & 1.55(0.12) & 0.54(0.31) &
        1.17(0.11) & 134.12(58.92) &  29.94(4.57)\\ \hline
    BO &0.83 (0.84) & 0.16(0.24) & 1.44(0.76) & 1.09(0.50) & 0.22(1.05) &
       0.45(0.41) & 108.18(67.08) & 3217.40(157.31)\\ \hline
    MSM & 0.65(0.16) & 0.26(0.19) & 1.50(0.59) & 1.01(0.57)& 0.14(0.13) &
        0.51(0.15)  & 89.17(33.20) &  25.46(8.26)\\ \hline
   
  \end{tabular}
  \end{adjustbox}
  \label{tab:table-sup-blowfly}
  \end{center}
  \vspace{-0.25in}
\end{table}

Table \ref{tab:table-sup-blowfly} shows the separate errors made by each method for individual parameters. This was omitted from the main text due to space constraints.

\subsection{Prior distribution for the parameters of the blowfly population dynamics}

We describe here the prior distribution for the parameters $\theta := (P \in \mathbb{N},\ N_{0} \in \mathbb{N},\ \sigma_{d} \in \mathbb{R}_{+},\ \sigma_{p} \in \mathbb{R}_{+},\ \tau \in \mathbb{N},\ \delta\ \in \mathbb{R}_{+})$ in the blowfly population dynamics, the parameters that we used in our experiment.
Let $\epsilon_p,\epsilon_{N_0}, \epsilon_{\sigma_d},\epsilon_{\sigma_p},\epsilon_{\tau},\epsilon_{\delta} \sim \rm{Normal}(0,1)$ be independent standard Gaussian random variables.
The prior can then be specified by defining the parameters as such random variables as
\begin{eqnarray*}
P &=&  \exp(2 + 2 \epsilon_p),\\
N_0 &=& \exp(5 + 0.5 \epsilon_{N_0}),\\
\sigma_d &=& \exp(-0.5 +  \epsilon_{\sigma_d}),\\
\sigma_p &=& \exp(-0.5 +  \epsilon_{\sigma_p}),\\
\tau &=& \exp(2 + \epsilon_{\tau}),\\
\delta &=& \exp(-1 + 0.4 \epsilon_{\delta}).
\end{eqnarray*}
Note that the parameters $P, N_0, \tau$ are to be rounded appropriately, as they are defined as being natural numbers.

\section{Supplementary materials for the experiments on alpha stable distributions in Section 4.4}
\subsection{Computation time}

We offer here supplementary materials for the experiment on multivariate alpha stable distributions in Section 4.4 of the main text.

\begin{comment}
\begin{figure}[H]
\begin{minipage}{1\hsize}
\begin{center}
\includegraphics[scale = 0.7]{error_decrease_example.png}
\end{center}
\end{minipage}
\label{fig:alpha-example}
\end{figure}

Figure \ref{fig:alpha-example} shows that in fact the recursive application of the Bayes rule via the proposed algorithm decreased error.
\end{comment}

\begin{figure}[H]
\vspace{.3in}
\begin{center}
\includegraphics[scale = 0.2]{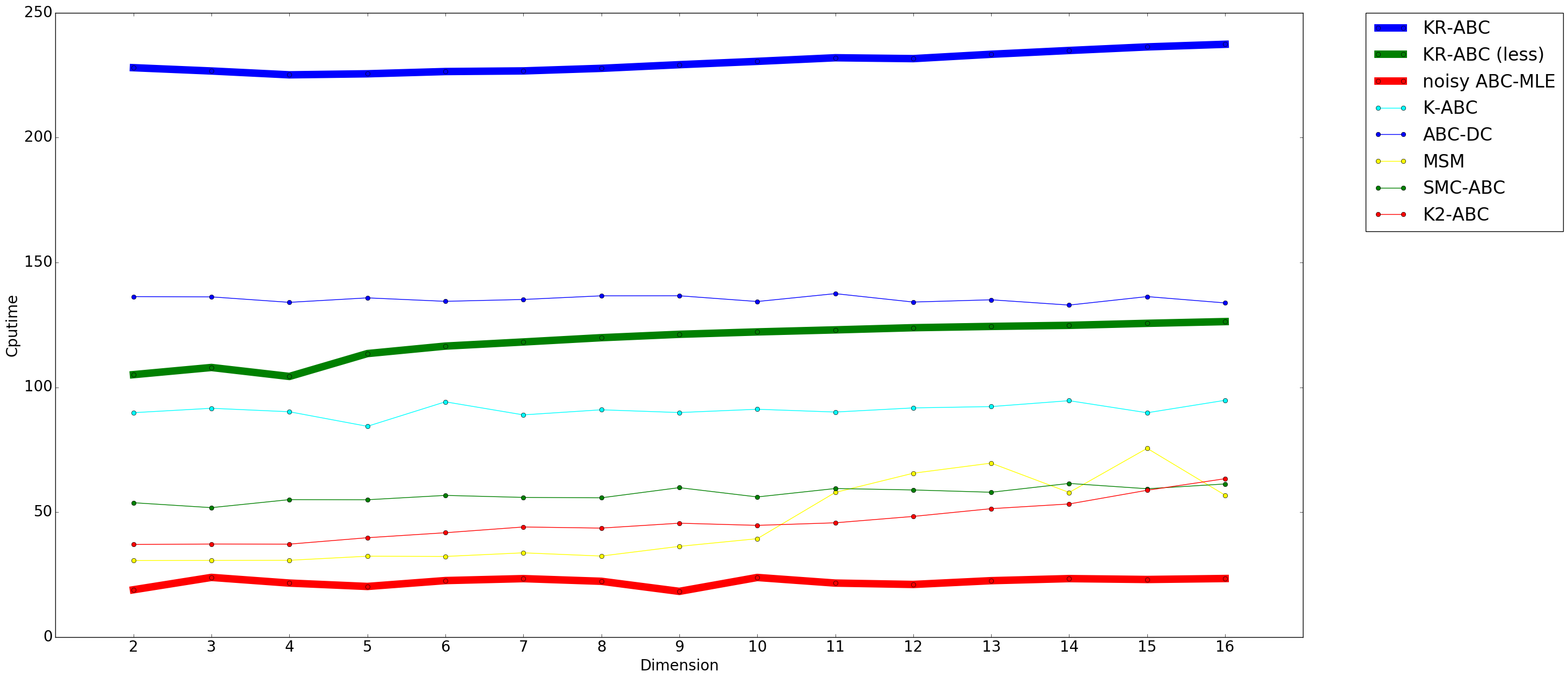}
\end{center}
%\vspace{.4in}
\label{fig:alpha_cpu}
\caption{Computation time (in seconds) for the experiments in Section 4.4. We omit here the computation time of Bayesian optimization, but it was more than 1500 seconds for all the dimensions.}
\end{figure}

Figure \ref{fig:alpha_cpu} shows computation time for each method in the experiments on multivariate alpha stable distributions in Section 4 of the main text, which information was omitted from the main body because of space constraints.

%For illustrative purposes, the computation time of Bayesian optimization is omitted from the figure. 
%Their cputime was [1934.46,  1706.44, 2149.78,  2496.34, 2334.47, 2342.88, 2586.13, 2877.15, 2728.37, 2538.57, 2259.93, 2664.31, 2795.95, 3052.02, 2074.48]

\subsection{Definition of the deterministic map for sampling}
We describe here the deterministic map $\tau_\theta$ used in sampling multivariate alpha stable distributions \citep{chambers1976method}, where $\theta := (\alpha, \beta, \mu, \sigma) \in (0,2] \times [-1,1] \times \mathbb{R} \times [0,\infty)$.
Given $U_1 \sim {\rm Unif}(-\pi/2, \pi/2)$ and $U_2 \sim {\rm Exp}(1)$, the mapping $\tau_\theta(U_1,U_2) \in \R$ is defined as
\begin{eqnarray*}
\tau_{\theta}(U_1,U_2) := \sigma\tau_{\alpha,\beta}(U_1,U_2) + \mu,
\end{eqnarray*}
where

~~~~~~~~~~~~$\tau_{\alpha,\beta}(U_1,U_2) := \left\{ 
\begin{tabular}{@{}l@{}}
$S_{\alpha,\beta}\frac{\sin[\alpha(\it{U_1} + B_{\alpha,\beta})]}{[\cos(\it{U_1})]^{1/\alpha}} (\frac{\cos[\it{U_1} - \alpha(U_1 + B_{\alpha,\beta})]}{U_2})^{(1-\alpha)/\alpha}, ~~ \alpha \neq 1$ \\
$X = \frac{2}{\pi}[(\frac{\pi}{2} + \beta\it{U_1})\tan\it{U_1} - \beta\rm{log}(\frac{\it{U_2}\cos\it{U_1}}{\frac{\pi}{2}+\beta\it{U_1}})], ~~ \alpha = 1.$
\end{tabular} \right\}$\\
\begin{eqnarray*}
B_{\alpha,\beta} := \frac{\tan^{-1}(\beta\tan\frac{\pi\alpha}{2})}{\alpha},~~~~~S_{\alpha,\beta} := \left(1+\beta^2\tan^2\frac{\pi\alpha}{2} \right)^{1/2\alpha} .
\end{eqnarray*}

\begin{comment}
\begin{eqnarray*}
A \sim {\bf S}(\alpha/2, 1, 2({\rm cos}\pi\alpha/4)^{2/\alpha},0),\\
A = \tau_{\theta}(U),\\
\theta = \{\alpha, \beta, \mu, \sigma\} \in (0,2] \times [-1,1] \times \mathbb{R} \times [0,\infty),\\
U = \{U_1, U_2\},
\end{eqnarray*}
\end{comment}

\section{Supplementary material for the pedestrian simulator experiment in Section 4.6}

We present here supplementary material w.r.t. the pedestrian simulator experiment in Section 4.6.

%\newpage
\subsection{Example of simulation results obtained with CrowdWalk}

\begin{figure}[h]
%\vspace{.3in}
\centerline{\includegraphics[scale = 0.65]{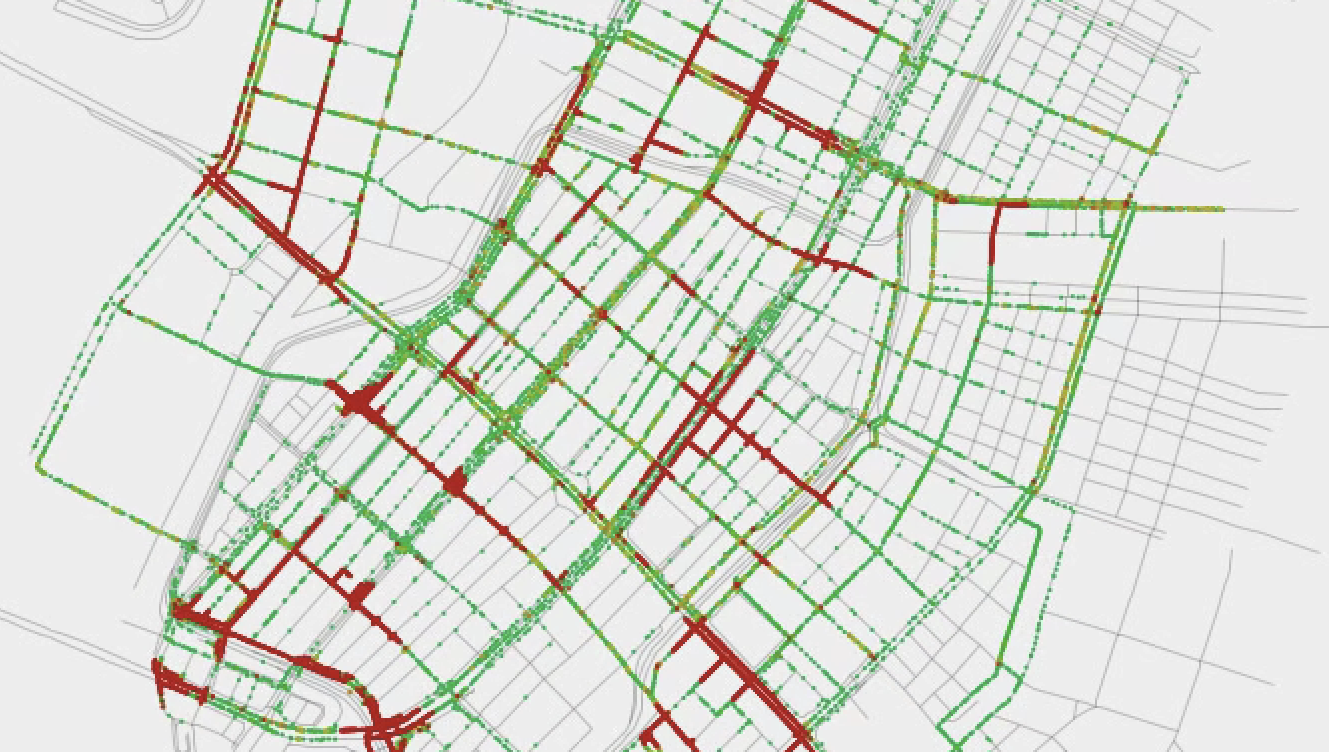}}
%\vspace{.3in}
\caption{Example of simulation results obtained with Crowdwalk}
\label{fig:figure-crowdwalk}
\end{figure}

Figure \ref{fig:figure-crowdwalk} shows an example of simulation results with the pedestrian flow simulator CrowdWalk \citep{Noda2010p}.
The map is of Ginza, one of the largest commercial districts in Tokyo.
Both green and red points indicate pedestrians, each of which is moving at an individual speed.
Red points are pedestrians who are walking particularly slowly; these pedestrians are forced to walk slowly because the areas in which they are walking are crowded.

\subsection{Parameters for the pedestrian simulator}

\begin{table}[H]
  \caption{Parameters for the pedestrian simulator, including the fixed ones}
  \begin{center}
  \begin{tabular}{|c||c|c|c|c|c|c|} \hline
Group & $\theta^{(S)}$ & $\theta^{(G)}$ & $\theta^{(P)}$ & $\theta^{(R)}$ & $\theta^{(N)}$ & $\theta^{(T)}$ \\\hline \hline
1 & 1	&0 & 15, 29	& 5400, 1800	&100	&30 \\\hline
2 & 5	&2 & 24, 43	& 5400, 5400	&100	&60 \\\hline
3 & 8	&3 & 28, 48	& 5400, 5400	&100	&90 \\\hline
4 & 4	&7 & 2, 0	& 1800, 3600	&100	&120 \\\hline
5 & 3	&2 & 8, 9	& 5400, 5400	&100	&150 \\\hline
6 & 6	&9 & 26, 14	& 5400, 3600	& 0	& 180 \\\hline
7 & 10 &11 & 4, 41	& 1800, 3600	& 0	&210 \\\hline
8 & 2	&9 & 50, 18	& 1800, 3600	&0	&240 \\\hline
9 & 0	&6 & 40, 33	& 3600, 5400	&0	&270 \\\hline
10 & 11	&5 & 20, 25	& 1800, 3600	&0	&300 \\\hline
  \end{tabular}
  \end{center}
  \label{tb:para-pedestrian}
\end{table}

Table \ref{tb:para-pedestrian} shows the parameters of the 10 candidate groups in a mixture model used for parameter estimation. 
Note that the parameters $\theta^{(N)}$ and $\theta^{(T)}$ were {\em unknown} for each method since they were the parameters to be estimated.
Groups 1 to 5 are the components of the true model, but this fact was also unknown for each method.
The numbers in $\theta^{(S)}$, $\theta^{(G)}$ and $\theta^{(P)}$ indicate certain locations on the map (e.g., the Mitsukoshi Department Store, the Apple Store, and Ginza Station), which are predefined in terms of two-dimensional coordinates. 
The parameter $\theta^{(P)}$ indicates certain places where pedestrians in a single group visit. 
In this experiment, pedestrians in each group visited 2 intermediate places during the travel from the starting location to the goal; $\theta^{(R)}$ represent the respective durations of time (in seconds) at the intermediate places.
(Note that the units for starting time $\theta^{(T)}$ are in minutes.)
%Note also that the parameters of the bottom 5 groups have no meaning since the respective number of the agents for each is 0.

\subsection{Estimated parameters with the proposed method}
Tables \ref{tab:Table4} and \ref{tab:Table5} show the estimated values for the parameters $\theta_i^{(N)}$ and  $\theta_i^{(T)}$, respectively, for each of independent 20 trials.
Recall that $i$ in $\theta_i^{(N)}$ and  $\theta_i^{(T)}$ is the index of 10 groups, i.e,, $i = 1,\dots,10$.
Results show that the proposed method was able to estimate the parameters of the 5 true groups in most cases.
Note that the estimated values of $\theta_i^{(T)}$ for $i = 6,\dots,10$ were rather arbitrary. This is reasonable since the corresponding numbers of pedestrians $\theta_i^{(N)}$ in these groups were estimated to be zero or very small, and thus these groups could be treated as being nonexistent.

\begin{table}[h]
  \caption{Estimated values of $\theta_i^{(N)}$ with KR-ABC for each of 20 trials}
  \label{tab:Table4}
  \begin{center}
  \scalebox{0.95}{
  \begin{tabular}{|l||l|c|r|r|r|r|r|r|r|r|} \hline 
Trial & $\theta_1^{(N)}$ & $\theta_2^{(N)}$ & $\theta_3^{(N)}$ & $\theta_4^{(N)}$ & $\theta_5^{(N)}$ & $\theta_6^{(N)}$ & $\theta_7^{(N)}$ & $\theta_8^{(N)}$ & $\theta_9^{(N)}$ & $\theta_{10}^{(N)}$\\\hline \hline
1&  95	&132 &102	&1	&79	&32&	2	&9	&13	&31\\\hline
2&100	&105 &105	&84	&100&	0	&0&	0	&0	&0\\\hline
3&98	&102 &90	&81	&101	&0	&16	&0	&7	&0\\\hline
4&93	&98	 &99	&101&	95&	0	&0	&0	&0	&9\\\hline
5&103	&12  & 84	&9	&88&	27&	2	&33&	19&	3\\\hline 
6&87	&105	&108	&100	&92	&0	&6	&0	&0	&0\\\hline
7&90	&102	&104	&102&	97&	0	&0&	0&	0&	0\\\hline
8&116	&79	&93	&118&	88	&0&	0	&0&	0&	1\\\hline
9&97	&91&	101	&110&	94&	0&	0&	1&	0	&0\\\hline
10&102	&127&	0	&0	&97	&0	&38&	70	&51&	12\\\hline
11&102	&105	&94&	87	&100&	0&	0&	0&	0&	7\\\hline
12&0	&2&	100&	228&	103&	0&	0&	0&	1&	64\\\hline
13&98	&97&	96&	108&	95&	0&	0&	0&	1&	0\\\hline
14&103	&98&	94	&94	&102&	3&	0&	0&	0&	1\\\hline
15&105	&176&	106&	0&	9&	78	&2&	14&	3&	2\\\hline
16&96	&134&	100&	0&	95&	0&	70&	0&	0&	0\\\hline
17&98	&106&	96&	87&	98&	4	&1&	4&	0&	2\\\hline
18&798	&101&	97&	97&	98&	0&	2&	1&	1&	0\\\hline
19&109&	54&	55	&181&	68&	0&	0&	31&	0&	0\\\hline
20&98	&90	&101&	102&	99&	3&	2	&0&	0&	0\\\hline	
  \end{tabular}}
  \end{center}
\end{table}
\begin{table}[H]
  \caption{Estimated values of   $\theta_i^{(T)}$ with KR-ABC for each of 20 trials}
  \label{tab:Table5}
  \begin{center}
  \scalebox{1.0}{
  \begin{tabular}{|l||l|c|r|r|r|r|r|r|r|r|} \hline
Trial & $\theta_1^{(T)}$ & $\theta_2^{(T)}$ & $\theta_3^{(T)}$ & $\theta_4^{(T)}$ & $\theta_5^{(T)}$ & $\theta_6^{(T)}$ & $\theta_7^{(T)}$ & $\theta_8^{(T)}$ & $\theta_9^{(T)}$ & $\theta_{10}^{(T)}$\\\hline \hline
1& 35 &	56& 	63& 	289 &	158 &	91 &	252& 	216& 	325 &	209\\\hline
2&25& 	52& 	87& 	121& 	152 &	186 &	152& 	22 &	193 &	460\\\hline
3&28 &	53 &	93 &	119 &	145 	&110 	&89& 	201 &	146 &	18\\\hline
4&29& 	60& 	92& 	120 &	147 &	356 &	188& 	147 &	249 &	0\\\hline
5&27 &	66 &	88& 	229 &	138 &	85 &	130& 	54& 	181 &	236\\\hline
6&22 &	53 	&85 	&121 &	151 	&338 	&208& 	309 	&31 	&135\\\hline
7&33& 	61& 	91 &	120 &	147 &	2 	&175& 	214& 	124 &	396\\\hline
8&26& 	68 &	95 &	129 &	136 &	25 &	375& 	161 &	81& 	266\\\hline
9&26& 	59& 	91 &	125 &	157 &	18 &	373& 	0 	&251 	&0\\\hline
10&30 &53 &452& 	243 &	151 	&285 	&69 &	89 	&175 &	216\\\hline
11&30& 	57 &	92 &	111 &	153 	&279 &	158& 	369& 	273 &	169\\\hline
12&213 &	173 &	92 	&125 &	154 &	130 	&77& 	0 &	456 &	0\\\hline
13&29& 	54 &	93 &	125 &	152 &	346 &	294 	&327& 	214 &	490\\\hline	
14&34& 	59&   89&	116& 	150&  275&  36&	490&  37&	109\\\hline
15&28& 	54& 	88& 	135 &	362& 	70& 	0 &	319 &	456 &	0\\\hline
16&29& 	50& 	88 &	285 &	146 &	0 &	403 	&98 &	286 &	300\\\hline
17&27& 	54& 	86 &	121 &	149 &	22 &	72 &	310& 	21 &	93\\\hline
18&30& 	59& 	89 &	119 &	146 &	221& 	107& 	361 &	218 &	260\\\hline
19&27 &	74 &	101 &	119 	&123 &	211 &	272 &	206 &	265 &	275\\\hline
20&30 &	54 &	87 &	127 &	146 	&286 	&452 	&199& 	267 &	119\\\hline
  \end{tabular}}
  \end{center}
\end{table}

\section{Linear time estimator for the energy distance}

In a way similar to that with \citet[Section 6]{gretton2012kernel}, we define here a linear-time estimator for the energy distance \citep{szekely2013energy}.
Let $x_1\dots,x_n \sim P$ and $y_1,\dots,y_n \sim Q$ be i.i.d.~samples from the two distributions $P$ and $Q$, and let $n_2 := \lfloor n/2 \rfloor$.
The linear estimator can then be defined as
$$
\frac{1}{n_2} \sum_{i = 1}^{n_2} h( (x_{2i-1}, y_{2i-1}), (x_{2i}, y_{2i}) ),
$$
where 
$$ h( (x_{2i-1}, y_{2i-1}), (x_{2i}, y_{2i}) )  
:= \| x_{2i-1} - y_{2i} \| + \| x_{2i} - y_{2i-1} \|  - \| x_{2i-1} - x_{2i} \| - \| y_{2i-1} - y_{2i} \| .
$$
It can be easily shown that this is unbiased and converges to the population energy distance between $P$ and $Q$ at a rate of $O_p(n^{-1/2})$, as \citet[Theorem 15]{gretton2012kernel} showed for a linear estimator in MMD. 
The above linear estimator can be computed at a cost of $O(n)$, which is less than the cost of $O(n^2)$ required for an ordinary quadratic estimator. (Note, however, that the linear estimator has higher variance than a quadratic one.)

\end{document}